\documentclass{article}
\usepackage{iclr2027_conference,times}

\usepackage{microtype}
\usepackage{graphicx}
\usepackage{subcaption}  
\usepackage{booktabs} 
\usepackage{array}

\usepackage{hyperref}
\hypersetup{
    colorlinks=false,
    pdfborder={0 0 1.5},
    linkbordercolor={1 0 0},
    citebordercolor={0 1 0},
    urlbordercolor={0 1 1}
}

\usepackage{amsmath}
\usepackage{amssymb}
\usepackage{mathtools}
\usepackage{amsthm}

\usepackage[capitalize,noabbrev]{cleveref}

\usepackage[utf8]{inputenc} 
\usepackage[T1]{fontenc} 
\usepackage{url} 
\usepackage{amsfonts} 
\usepackage{nicefrac} 
\usepackage{xcolor} 
\usepackage{mathrsfs}
\usepackage{enumerate}
\usepackage{wrapfig}
\usepackage{float}
\usepackage{bbding}
\usepackage{multirow}
\usepackage[misc]{ifsym}
\usepackage{placeins}
\usepackage{import}
\usepackage{tikz}
\usetikzlibrary{graphs}
\usetikzlibrary{positioning}

\theoremstyle{plain}
\newtheorem{theorem}{Theorem}

\newtheorem{lemma}{Lemma}
\newtheorem{corollary}{Corollary}
\theoremstyle{definition}
\newtheorem{definition}{Definition}
\theoremstyle{remark}
\newtheorem{remark}{Remark}

\numberwithin{theorem}{section}
\numberwithin{proposition}{section}
\numberwithin{lemma}{section}
\numberwithin{corollary}{section}
\numberwithin{definition}{section}
\numberwithin{remark}{section}

\newenvironment{proofsketch}{\par\textit{Proof Sketch}.\quad}{}

\usepackage{amsmath}
\usepackage{amsfonts}
\usepackage{bm}
\usepackage{esint}

\def\eqref#1{equation~\ref{#1}}

\def\floor#1{\lfloor #1 \rfloor}
\def\1{\bm{1}}

\def\rva{{\mathbf{a}}}
\def\rvb{{\mathbf{b}}}
\def\rvc{{\mathbf{c}}}

\def\rvh{{\mathbf{h}}}
\def\rvu{{\mathbf{i}}}

\def\rvu{{\mathbf{u}}}

\def\rvx{{\mathbf{x}}}
\def\rvy{{\mathbf{y}}}
\def\rvz{{\mathbf{z}}}

\def\vb{{\bm{b}}}
\def\vc{{\bm{c}}}

\def\vv{{\bm{v}}}
\def\vw{{\bm{w}}}

\def\vz{{\bm{z}}}

\def\mH{{\bm{H}}}
\def\mI{{\bm{I}}}

\def\mO{{\bm{O}}}
\def\mP{{\bm{P}}}
\def\mQ{{\bm{Q}}}

\def\mV{{\bm{V}}}
\def\mW{{\bm{W}}}
\def\mX{{\bm{X}}}

\DeclareMathAlphabet{\mathsfit}{\encodingdefault}{\sfdefault}{m}{sl}
\SetMathAlphabet{\mathsfit}{bold}{\encodingdefault}{\sfdefault}{bx}{n}
\newcommand{\tens}[1]{\bm{\mathsfit{#1}}}

\def\tG{{\tens{G}}}

\def\tX{{\tens{X}}}

\def\tZ{{\tens{Z}}}

\def\gB{{\mathcal{B}}}

\def\gF{{\mathcal{F}}}

\def\gS{{\mathcal{S}}}

\def\sR{{\mathbb{R}}}

\def\sZ{{\mathbb{Z}}}

\DeclareMathOperator{\sign}{sign}

\newcommand{\Eqn}{Eqn.}

\newcommand{\mat}[1]{\begin{bmatrix}#1\end{bmatrix}}

\title{Parallel Layer Normalization for Universal Approximation}
\author{%
\makebox[\dimexpr\textwidth-2\tabcolsep\relax][c]{%
\begin{tabular}{@{}c@{}}
Yunhao Ni$^{1}$ \quad Yuxin Guo$^{1}$ \quad Yuhe Liu$^{1}$ \quad Wenxin Sun$^{1}$\\[0.2ex]
Jie Luo$^{1}$ \quad Wenjun Wu$^{1,2}$ \quad Lei Huang$^{1,2}$\thanks{Corresponding author: \texttt{huangleiAI@buaa.edu.cn}.}\\[0.7em]
{\normalfont $^{1}$State Key Laboratory of Complex and Critical Software Environment (SKLCCSE),}\\
{\normalfont Beihang University, Beijing, China}\\[0.2ex]
{\normalfont $^{2}$Hangzhou International Innovation Institute, Beihang University, Hangzhou, China}
\end{tabular}}%
}

\iclrfinalcopy 
\begin{document}

\maketitle
\lhead{}

\begin{abstract}
Layer Normalization (LN) is usually studied for its role in optimization, but
it also acts as a multivariate nonlinearity between affine layers. We investigate
how grouping LN units affects the approximation power of neural networks as
width grows. In Parallel Layer Normalization networks (PLN-Nets), fixed-dimensional LN units
are applied independently to groups between two affine layers. We prove that a
shallow LN-Net with one LN unit is non-universal regardless of width, whereas
PLN-Nets become universal as the number of parallel groups grows. The mechanism
extends to deep networks, Root Mean Square Normalization (RMSNorm)
variants, position-wise feed-forward
networks, and convolutional networks under suitable grouping structures.
Controlled MLP and physics-informed neural network (PINN) experiments support
PLN's practical viability, while
CNN and ViT experiments show that group count shapes the finite-sample fitting
envelope but does not act alone: the grouped axes, replacement site, and label
structure lead to distinct fitting and generalization behavior.
\end{abstract}

    \section{Introduction}

Normalization has become a fundamental component of modern neural networks.
Batch Normalization (BN; \citealp{2015_ICML_Ioffe}) was originally introduced to
stabilize and accelerate the training of deep networks, while Layer
Normalization (LN; \citealp{2016_LN_Ba}) was developed to avoid dependence on
batch statistics. LN and its variants are now standard components of
Transformers \citep{2017_NIPS_Vaswani}, and group-based normalization operators
are widely used in convolutional and structured architectures
\citep{2018_ECCV_Wu}. Most existing studies explain normalization through its
effects on optimization, including scale invariance, conditioning, and training
stability. Its contribution to the representation capacity of neural networks,
however, remains substantially less understood.

A recent study \citep{ni2024nonlinearity} initiated this representation
perspective by treating LN as a multivariate nonlinearity. It established
finite-sample classification results for deep LN-Nets and showed that grouping
can amplify LN's nonlinearity. Finite-sample fitting and stronger nonlinearity,
however, do not establish universal approximation. It remains unclear whether
increasing the number of groups, rather than merely widening one LN unit, can
enable LN-Nets to approximate arbitrary continuous functions.

We answer this question first in multilayer perceptrons
(MLPs; \citealp{cybenko1989approximation,hornik1991approximation}). We define
Parallel Layer Normalization (PLN; Definition~\ref{def:PLN}) as a direct sum of
fixed-dimensional LN units placed between affine layers. A shallow network
with one LN unit is non-universal regardless of its width, whereas shallow
PLN-Nets become universal as the number of groups grows
(Theorems~\ref{thm:LN-Nets weak}--\ref{thm:shallow PLN->universal}). The same
construction yields deep-network approximation results with explicit bounds.
An exact function-class equivalence with Parallel Layer Scaling (PLS) extends
these results to Root Mean Square Normalization (RMSNorm)-based networks
\citep{2019_NIPS_Zhang}.

Convolutional neural networks (CNNs; \citealp{2015_Nature_Yann}) and
Transformers \citep{2017_NIPS_Vaswani} are more complex than MLPs, yet retain
the affine--nonlinear building blocks of MLPs.
We prove exact representation results for convolutional networks and
position-wise feed-forward networks: in both cases, affine layers share
parameters across positions while PLN acts independently on feature or channel
coordinates. Token and spatial dimensions also admit grouping domains that
span multiple positions, beyond the position-wise PLN structure covered by our
theory. We therefore investigate these other grouping schemes empirically.

We conduct regression experiments on MLPs to empirically demonstrate the
approximation capability of finite PLN-Nets, consistent with our
universal-approximation result. We also use PINN experiments to explore their
potential for physics-informed learning. We further investigate grouped normalization under different grouping schemes in
CNNs and Transformers, finding that its performance depends on the grouping
scheme and data structure. The detailed experimental settings and results are
presented in Section~\ref{sec:experiments}.

\FloatBarrier
\section{Related Work}
\label{sec:related work}

\subsection{Universal Approximation and Neural Network Expressivity}

The universal approximation property of neural networks has been studied for
more than three decades. Early results established that single-hidden-layer
networks with sigmoidal activations are universal approximators
\citep{cybenko1989approximation}, and were subsequently generalized to broad
classes of bounded nonconstant activations \citep{hornik1991approximation}.
Later work characterized approximation rates \citep{barron1993universal} and
examined architectural constraints including arbitrary depth
\citep{gripenberg2003approximation}, bounded depth and width
\citep{maiorov1999lower}, and minimal width
\citep{park2020minimum,shin2025minimumwidthuniversalapproximation}. Expressivity
has also been studied through linear regions \citep{montufar2014number} and VC
dimension \citep{bartlett2019nearly}. Universal approximation analyses have
further been extended from MLPs to CNNs
\citep{oono2023approximationnonparametricestimationresnettype}, RNNs
\citep{li2022approximation}, and Transformers
\citep{yun2020transformersuniversalapproximatorssequencetosequence}.

Many constructive UAT proofs exploit properties specific to the activation.
For example, the piecewise-linear structure of ReLU permits localized
constructions and explicit approximation rates
\citep{yarotsky2017errorboundsapproximationsdeep,yun2020transformersuniversalapproximatorssequencetosequence},
whereas smooth activations such as tanh require different error-control
arguments \citep{de2021approximation}. Normalization instead introduces a
coupled, multivariate nonlinear transformation \citep{ni2024nonlinearity}.
Consequently, most approximation analyses of structured architectures omit
normalization for analytical tractability. Our results treat normalization
itself as the intermediate nonlinear operator.

\subsection{Theoretical Understanding of Layer and Group Normalization}

Normalization has primarily been analyzed through optimization. Existing
theory attributes its effectiveness to scale invariance and stabilized
training dynamics \citep{2019_ICLR_Arora,2016_LN_Ba,2023_TPAMI_Huang}, as well
as improved conditioning of the optimization problem
\citep{2019_ICML_Cai,2020_NIPS_Daneshmand,2019_ICML_Ghorbani,2019_NIPS_Karakida,2022_NIPS_Lv,2018_NIPS_shibani}.
These analyses explain important training benefits, but the role of
normalization in representation and function approximation remains less
understood.

From a representation perspective, prior work identified the intrinsic
multivariate nonlinearity of LN and showed that LN-Nets with three
neurons per layer and $O(m)$ LN layers can shatter any set of $m$ labeled
samples \citep{ni2024nonlinearity}. That result obtains classification capacity through depth. In
contrast, we study function approximation through width: we ask how arranging
multiple normalization units in parallel changes the representation class and
when normalization itself can provide universal approximation.

Grouping introduces an additional structural choice. Standard Group
Normalization (GN; \citealp{2018_ECCV_Wu}) partitions image channels and computes
each group's statistics jointly over spatial positions. Other designs
alter the normalization domain: PONO normalizes channels independently at each
spatial position \citep{li2019positional}, while LN-G-Position partitions those
channels further \citep{ni2024nonlinearity}. Related position-wise grouped
RMS-style normalizers also appear in RetNet \citep{sun2023retentive} and Mamba-2
\citep{dao2024transformers}, whereas temporal GN variants may pool statistics
across sequence positions \citep{baevski2020wav2vec2,chi2023diffusionpolicy}.
Despite their similar names, these operators lack a unified naming convention and may perform substantially different computations.

We formalize PLN as a
direct sum of fixed-dimensional LN units whose number grows with width, derive
its representation consequences, and extend the construction to sequence and
image tensors.

\section{Preliminaries and Parallel Layer Normalization}
\label{sec:preliminaries}

We first introduce the network notation used in the approximation results,
then define RMS normalization, LN, and their parallel forms, and finally
establish two structural properties of the resulting networks. We use lowercase letters for
scalars, boldface lowercase letters for vectors, and boldface uppercase letters
for matrices. The notation $[\rvx]_i$ denotes the $i$-th coordinate of a vector.

\subsection{Neural Networks}

\paragraph{Affine layers.}
An affine layer $\varphi:\sR^d\to\sR^m$ is of the form
$\varphi(\rvx)=\mW\rvx+\vb$, where $\mW \in \sR^{m \times d}$ and
$\vb \in \sR^{m}$. For consistency
with common deep-learning terminology, we continue to call such maps linear
layers when no confusion can arise.

\paragraph{Networks with intermediate operators.}
We consider feed-forward networks obtained by composing affine layers with a
class of intermediate operators $\phi$.

\begin{definition}[$\phi$-networks]
\label{def:networks} 
A $\phi$-network (or $\phi$-Net) of depth $L$ is a function  
$f : \mathbb{R}^{d_x} \to \mathbb{R}^{d_y}$ of the form
\begin{equation}
\label{eqn:phi-net-circ}
f
=
\varphi_{L+1} \circ \phi_L \circ \varphi_L \circ \cdots \circ \phi_1 \circ \varphi_1,
\end{equation}
where the $\varphi_\ell$ are affine layers of compatible dimensions and each
$\phi_\ell:\sR^{d_\ell}\to\sR^{d_\ell}$ belongs to the same operator class
$\phi$.
\end{definition}

\begin{remark}
If $\phi:\sR\to\sR$ is an element-wise activation and
$\rvx=(x_1,\ldots,x_d)^\top\in\sR^d$, we write its coordinate-wise action as
$\phi(\rvx):=(\phi(x_1),\ldots,\phi(x_d))^\top$. In contrast, a normalization
operator maps several coordinates jointly. Definition~\ref{def:PLN} applies
such fixed-dimensional normalization maps independently to multiple blocks.
\end{remark}

\paragraph{Depth and Width.}
The depth of \Eqn\ref{eqn:phi-net-circ} is $L$ and its width is
$\max\{d_1,\ldots,d_L\}$. We denote by $\mathcal F(\phi;L,N)$ the class of
$\phi$-networks of depth $L$ and width at most $N$. The case $L=1$ is called
shallow. For simplicity, throughout the paper, function classes compared by
inclusion or equality are taken to have the same input and output spaces.


\subsection{Parallel Layer Normalization}
\label{sec:PLN}

\paragraph{RMS normalization and Layer Normalization.}
For $\rvh\in\sR^d$, define the affine-free RMSNorm
core~\citep{2019_NIPS_Zhang}, the centering projection, and LN
\citep{2016_LN_Ba} by
\begin{equation}
\label{eqn:normalization operators}
 \mathrm{LS}(\rvh)
 =\rvh/\sqrt{d^{-1}\|\rvh\|_2^2},\qquad
 \mP_d=\mI_d-d^{-1}\bm1\bm1^\top,\qquad
 \mathrm{LN}(\rvh)
 =\mathrm{LS}(\mP_d\rvh).
\end{equation}
Here LS stands for layer scaling, our shorthand for the RMS normalization
operation. We complete the definition at zero denominator by setting
$\mathrm{LS}(\bm0)=\bm0$, which also gives $\mathrm{LN}(c\bm1)=\bm0$. We write
$\mathrm{LS}[d]$ and $\mathrm{LN}[d]$ for these maps on $\sR^d$, omitting the
dimension when clear. Practical implementations typically include learnable
coordinate-wise scale and shift parameters \citep{2016_LN_Ba} and add a small number
$\varepsilon>0$\footnote{The
stabilized definition and its consequences for our
theoretical results are discussed in Appendix~\ref{sec:proof thm 3.2-2}.} to
each denominator. We omit the affine parameters because they can be absorbed
into an adjacent affine layer.

PLN partitions the feature coordinates into groups and applies LN to all groups
in parallel.\footnote{The name PLN distinguishes this vector construction from
Group Normalization (GN), which is defined for convolutional feature maps and
computes statistics over each channel group and its spatial positions
\citep{2018_ECCV_Wu}. Here, each group contains feature coordinates only.}

\begin{definition}[Parallel Layer Normalization]
\label{def:PLN}
Let $d=n_sN$ and write
$\rvh=\rvh_1\oplus\cdots\oplus\rvh_N\in\sR^d$, where $\oplus$ denotes
concatenation and
$\rvh_i\in\sR^{n_s}$. Parallel Layer
Normalization with norm size $n_s$ is
\begin{equation}
    \mathrm{PLN}[n_s](\rvh)
    =\bigoplus_{i=1}^{N}\mathrm{LN}[n_s](\rvh_i)
    =\left[\mathrm{LN}[n_s](\rvh_1)^\top,\ldots,
    \mathrm{LN}[n_s](\rvh_N)^\top\right]^\top.
\end{equation}
Replacing LN by LS in the same direct sum defines Parallel Layer Scaling,
$\mathrm{PLS}[n_s]$. We fix $n_s$ and allow the number of groups $N$ to grow
with width.
\end{definition}

\subsection{Structural Properties of PLN Networks}

Following Definition~\ref{def:networks}, networks using $\mathrm{LN}[n_s]$ and
$\mathrm{PLN}[n_s]$ as intermediate operators are called LN-Nets and PLN-Nets,
respectively. The direct-sum structure in Definition~\ref{def:PLN} becomes a
sum after the output affine layer. Thus, a shallow PLN-Net is precisely a
parallel composition of independently parameterized, fixed-dimensional
LN-Nets.

\noindent
\begin{minipage}[t]{0.70\textwidth}
    \vspace{0pt}
    \begin{lemma}
    \label{lemma:PLN->LN}
    Given $N$ shallow LN-Nets
    $\hat f_i\in\gF(\mathrm{LN}[n_s];1,n_s)$ with
    $\hat f_i:\sR^{d_x}\to\sR^{d_y}$ for $1\le i\le N$, there exists a
    shallow PLN-Net $f\in\gF(\mathrm{PLN}[n_s];1,n_sN)$ such that
    \begin{equation}
    f(\rvx) = \sum_{i=1}^N \hat f_i(\rvx),\qquad
    \forall \rvx\in\sR^{d_x}.
    \end{equation}
    Conversely, any such PLN-Net can be decomposed into a sum of $N$ shallow
    LN-Nets with the same input and output dimensions.
    \end{lemma}

    Figure~\ref{fig:proof_pln} illustrates the scalar case; the formal proof
    for arbitrary input and output dimensions is given in
    Appendix~\ref{sec:proof 3.3}.
\end{minipage}\hfill
\begin{minipage}[t]{0.27\textwidth}
    \vspace{0pt}
    \centering
    \captionsetup{font=footnotesize,skip=2pt,hypcap=false}
    \includegraphics[width=\linewidth]{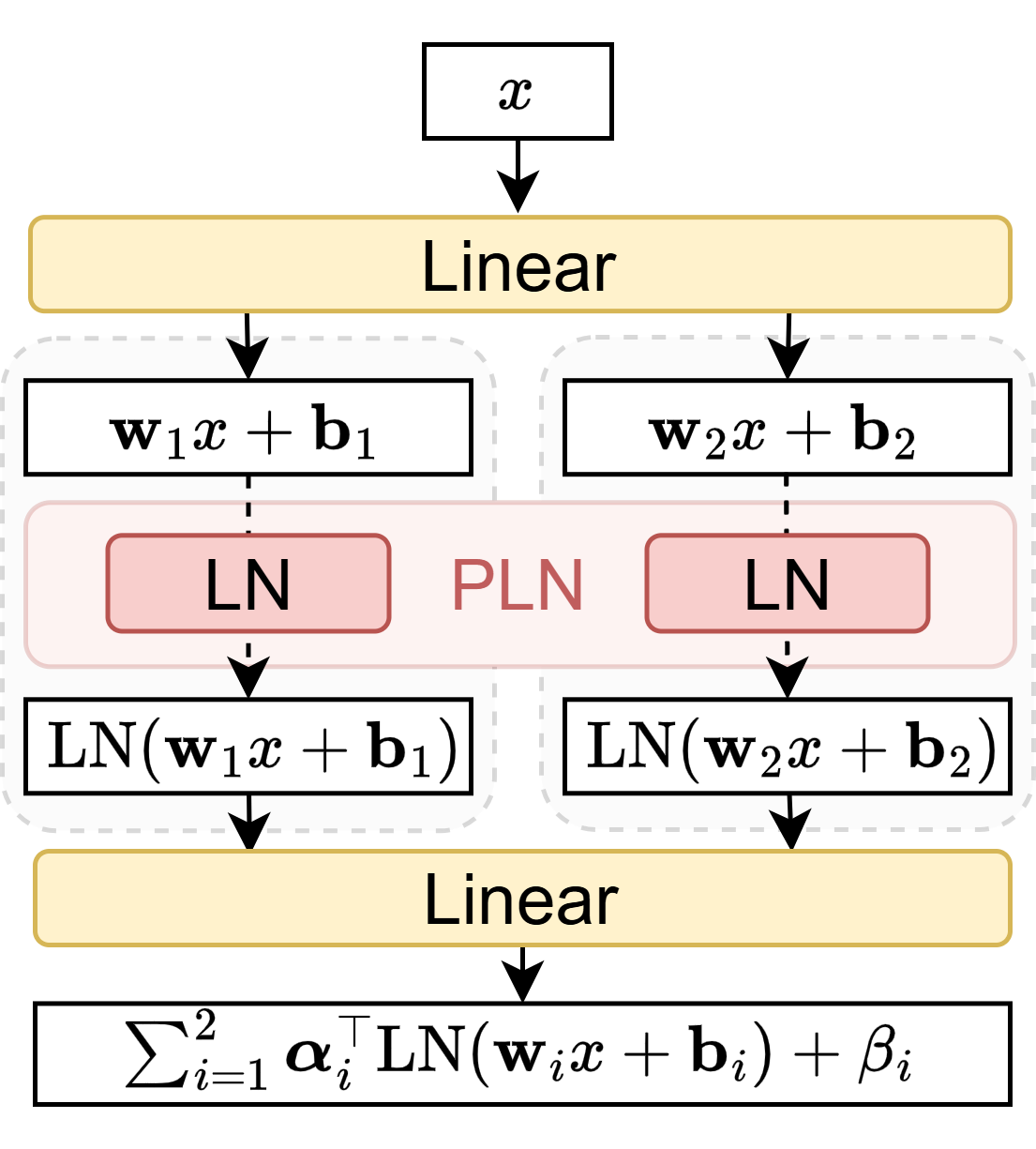}
    \captionof{figure}{PLN decomposition
    (Lemma~\ref{lemma:PLN->LN}).}
    \label{fig:proof_pln}
\end{minipage}
\par\medskip

Although the rank-$(d-1)$ centering projection $\mP_d$ is not invertible on
$\sR^d$, its image can be linearly identified with $\sR^{d-1}$. This yields
exact conversions between LN- and LS-Nets and, blockwise, between PLN- and
PLS-Nets.

\begin{lemma}
\label{lemma: PLN=PLS}
Shallow LN-Nets with width $n_s$ and shallow LS-Nets with width $n_s-1$ can represent each other:
\begin{equation}
\label{eqn:LN=LS}
\gF(\mathrm{LN}[n_s];1,n_s) = \gF(\mathrm{LS}[n_s-1];1,n_s-1),
\end{equation}
for $n_s\ge2$. Furthermore,
\begin{equation}
\gF(\mathrm{PLN}[n_s];1,n_sN) = \gF(\mathrm{PLS}[n_s-1];1,(n_s-1)N).
\end{equation}
\end{lemma}

The first identity is a single-block equivalence; the second follows by
combining it with the parallel decomposition in
Lemma~\ref{lemma:PLN->LN}. The proof is given in
Appendix~\ref{sec:proof 3.4}. Consequently, the subsequent
PLN representation and approximation results also apply to PLS whenever they
depend only on the represented function class, after shifting each normalization
block from width $n_s$ to $n_s-1$ and the corresponding total width from
$n_sN$ to $(n_s-1)N$. This equivalence also simplifies the subsequent
analysis: we may work with LS or PLS when convenient and transfer the resulting
conclusions to LN or PLN.

\section{Universal Approximation by PLN Networks}
\label{sec:extension}

We now establish the central representation result for MLPs. A single LN is a
multivariate nonlinearity and can realize useful scalar activations after
composition with affine layers. However, a shallow network with only one
full-width LN remains non-universal, because increasing its width enlarges one
coupled nonlinear unit while the number of units remains one. PLN removes
this obstruction by placing fixed-dimensional LN units in parallel. We first
make this contrast precise and then prove universal approximation for both
univariate and multivariate functions.

\subsection{Representation Capacity of Shallow LN-Nets}

We now characterize the scalar nonlinearities realized by one fixed-dimensional
LN unit, the basic component of a PLN-Net.

We first show that one $\mathrm{LN}[n_s]$ unit between two affine layers can
exactly represent\footnote{We use ``represent'' for exact realization (zero
approximation error), and ``approximate'' when a nonzero error is allowed.}
simple scalar activations, beginning with the sign function.

\noindent
\begin{minipage}[t]{0.72\textwidth}
    \vspace{0pt}
    \begin{lemma}
    \label{lemma:shallow LN->sign}
    Given any $\hat{f}\in\gF(\mathrm{sign};1,1)$ with $\hat f:\sR\to\sR$, there exists $f\in\gF(\mathrm{LN}[n_s];1,n_s)$ with $n_s\ge2$ such that $f=\hat f$. In short,
    \begin{equation}
    \gF(\mathrm{sign};1,1)\subseteq\gF(\mathrm{LN}[n_s];1,n_s),
    \end{equation}
    where $\mathrm{sign}(x)$ is the element-wise sign function taking values $1$ if $x>0$, $0$ if $x=0$, and $-1$ if $x<0$.
    \end{lemma}

    The proof is given in Appendix~\ref{sec:proof 3.1}. For the special case $n_s=2$, Figure~\ref{fig:proof_sign} illustrates the construction.

    The sign construction gives an explicit one-dimensional rate, but its
    derivative vanishes almost everywhere, hindering backpropagation
    \citep{amari1993backpropagation}. A wider LN unit ($n_s\ge3$) also realizes
    the smooth non-polynomial activation $\phi(x)=x/\sqrt{x^2+1}$ used for the
    general UAT.
\end{minipage}\hfill
\begin{minipage}[t]{0.25\textwidth}
    \vspace{0pt}
    \centering
    \captionsetup{font=footnotesize,skip=2pt,hypcap=false}
    \includegraphics[width=\linewidth]{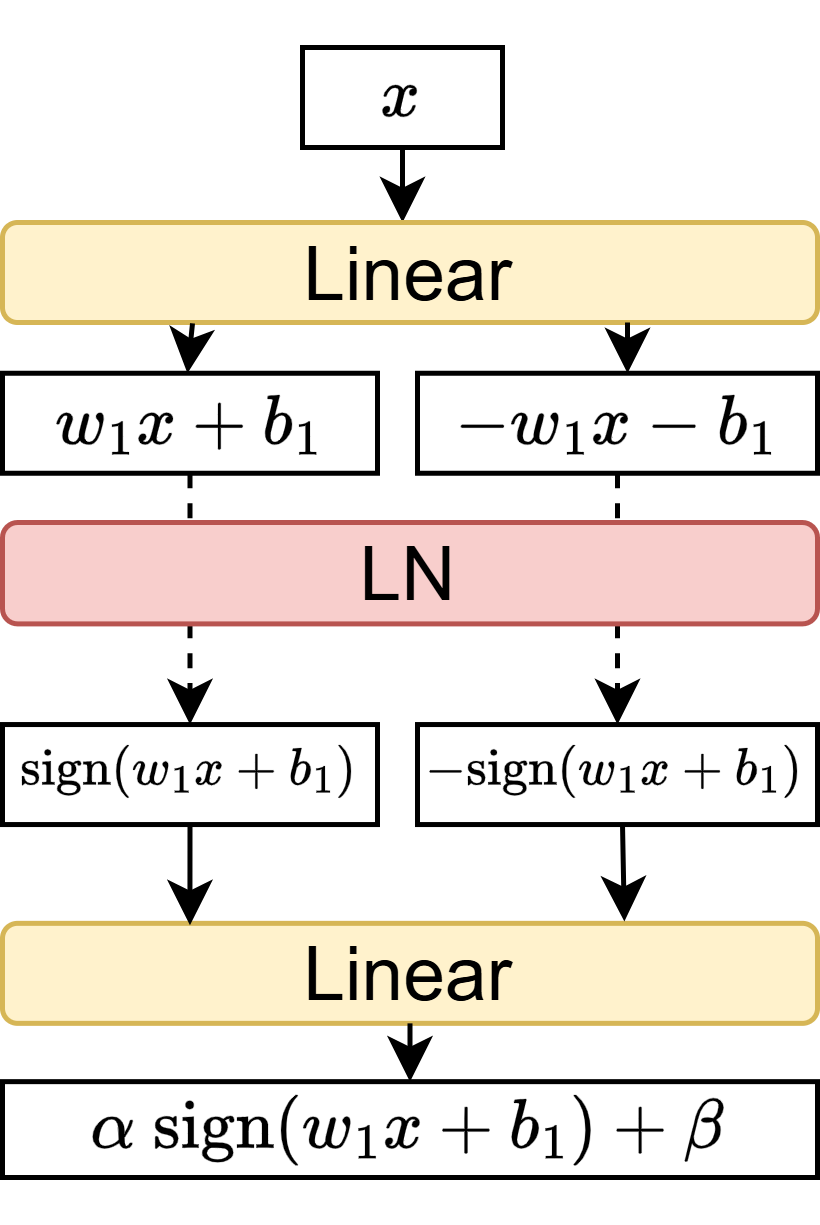}
    \captionof{figure}{Construction for
    Lemma~\ref{lemma:shallow LN->sign}.}
    \label{fig:proof_sign}
\end{minipage}
\par\vspace{2pt}

\noindent
\begin{minipage}[t]{0.72\textwidth}
    \vspace{0pt}
    \begin{lemma}
    \label{lemma:shallow LN->phi}
    Given any $\hat{f}\in\gF(\phi;1,1)$ with $\hat f:\sR\to\sR$, there exists $f\in\gF(\mathrm{LN}[n_s];1,n_s)$ with $n_s\ge3$ such that $f=\hat f$. In short,
    \begin{equation}
    \gF(\phi;1,1)\subseteq \gF(\mathrm{LN}[n_s];1,n_s),
    \end{equation}
    where $\phi(x)=x/\sqrt{x^2+1}$.
    \end{lemma}

    The proof is provided in Appendix~\ref{sec:proof 3.2}. Using Lemma~\ref{lemma: PLN=PLS}, it reduces to proving
    \begin{equation}
    \gF(\phi;1,1)\subseteq \gF(\mathrm{LS}[n_s-1];1,n_s-1).
    \end{equation}
    For $n_s=3$, Figure~\ref{fig:proof_phi} illustrates this construction.

    These scalar constructions do not imply universal approximation:

    \begin{theorem}[Weak representation capacity of shallow LN-Nets]
    \label{thm:LN-Nets weak}
    There exists a continuous function $\hat f$ on $\Omega\subset \sR$ such that
    every shallow LN-Net $f\in\gF(\mathrm{LN}[n_s];1,n_s)$ with $n_s\ge 2$
    satisfies $|f-\hat{f}|_{L^\infty(\Omega)}\ge 1$.
    \end{theorem}
\end{minipage}\hfill
\begin{minipage}[t]{0.25\textwidth}
    \vspace{0pt}
    \centering
    \captionsetup{font=footnotesize,skip=2pt,hypcap=false}
    \includegraphics[width=\linewidth]{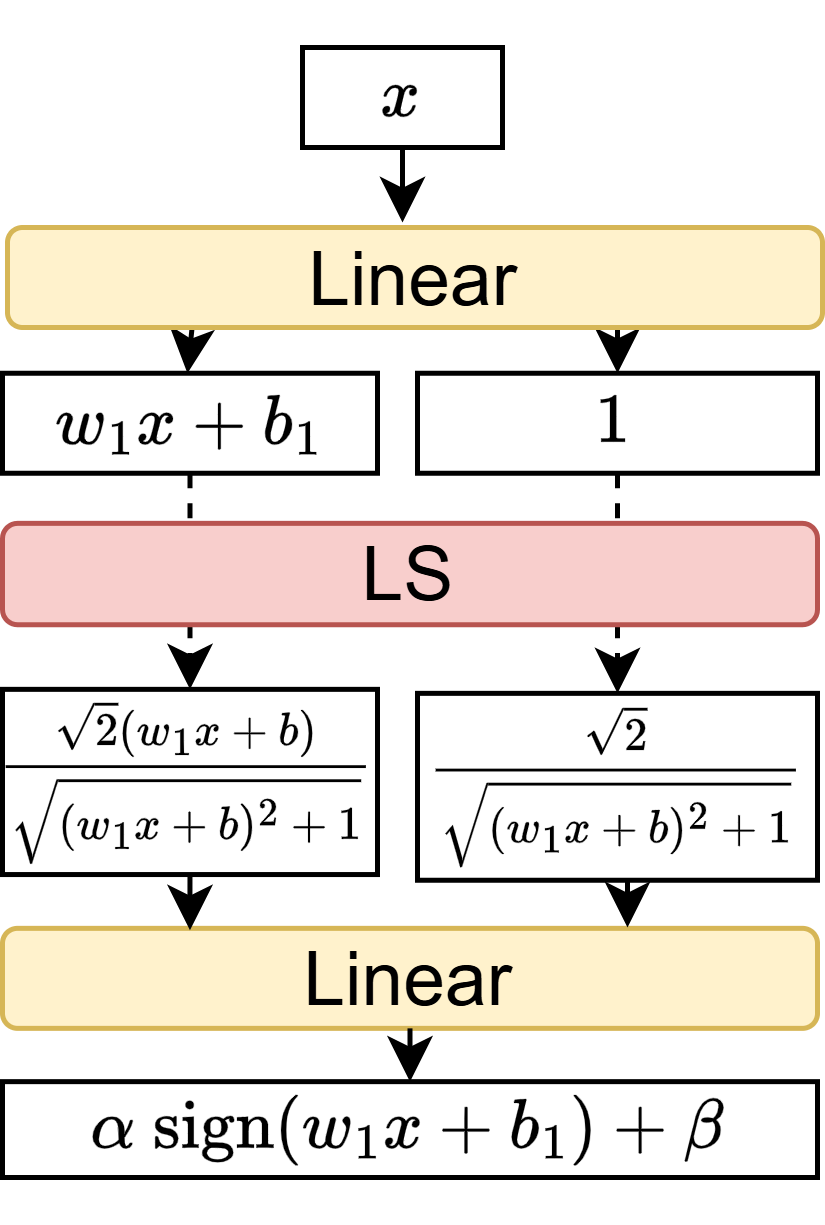}
    \captionof{figure}{Construction for
    Lemma~\ref{lemma:shallow LN->phi}.}
    \label{fig:proof_phi}
\end{minipage}
\par\vspace{2pt}

The counterexample used in the proof is the smooth elementary function $\hat f(x)=\cos(\pi x)$, showing that the limitation stems from the normalization structure. Please see Appendix~\ref{sec:proof thm 3.1} for the detailed proof.

Theorem~\ref{thm:LN-Nets weak} shows that increasing the dimension of one LN
unit is insufficient for shallow universal approximation. PLN instead fixes
the unit dimension and increases the number of units.

\subsection{Universal Approximation by PLN Networks}

The structural decomposition in Lemma~\ref{lemma:PLN->LN} lets the number of
fixed-dimensional nonlinear units grow with width. We first use it to lift
shallow activation networks to PLN-Nets.


\noindent
\begin{minipage}[t]{0.70\textwidth}
    \vspace{0pt}
    A shallow width-$N$ $\phi$-Net is a sum of $N$ width-one $\phi$-Nets.
    Together with Lemma~\ref{lemma:PLN->LN}, this gives the following
    corollary.

    \begin{corollary}[Representation of shallow activation networks]
    \label{coro:shallow PLN->phi}
    Suppose every width-one shallow $\phi$-Net can be represented by an
    $\mathrm{LN}[n_s]$-Net. Then, for any input and output dimensions and any
    $\hat f\in\gF(\phi;1,N)$, there exists
    $f\in\gF(\mathrm{PLN}[n_s];1,n_sN)$ such that $f=\hat f$. In short,
    \begin{equation}
    \gF(\phi;1,N)\subseteq\gF(\mathrm{PLN}[n_s];1,n_sN).
    \end{equation}
    \end{corollary}

    The construction assigns one LN block to each scalar activation unit; the
    complete proof is given in Appendix~\ref{sec:proof corollary 3.1}. This yields:
\end{minipage}\hfill
\begin{minipage}[t]{0.27\textwidth}
    \vspace{0pt}
    \centering
    \captionsetup{font=footnotesize,skip=2pt,hypcap=false}
    \includegraphics[width=\linewidth]{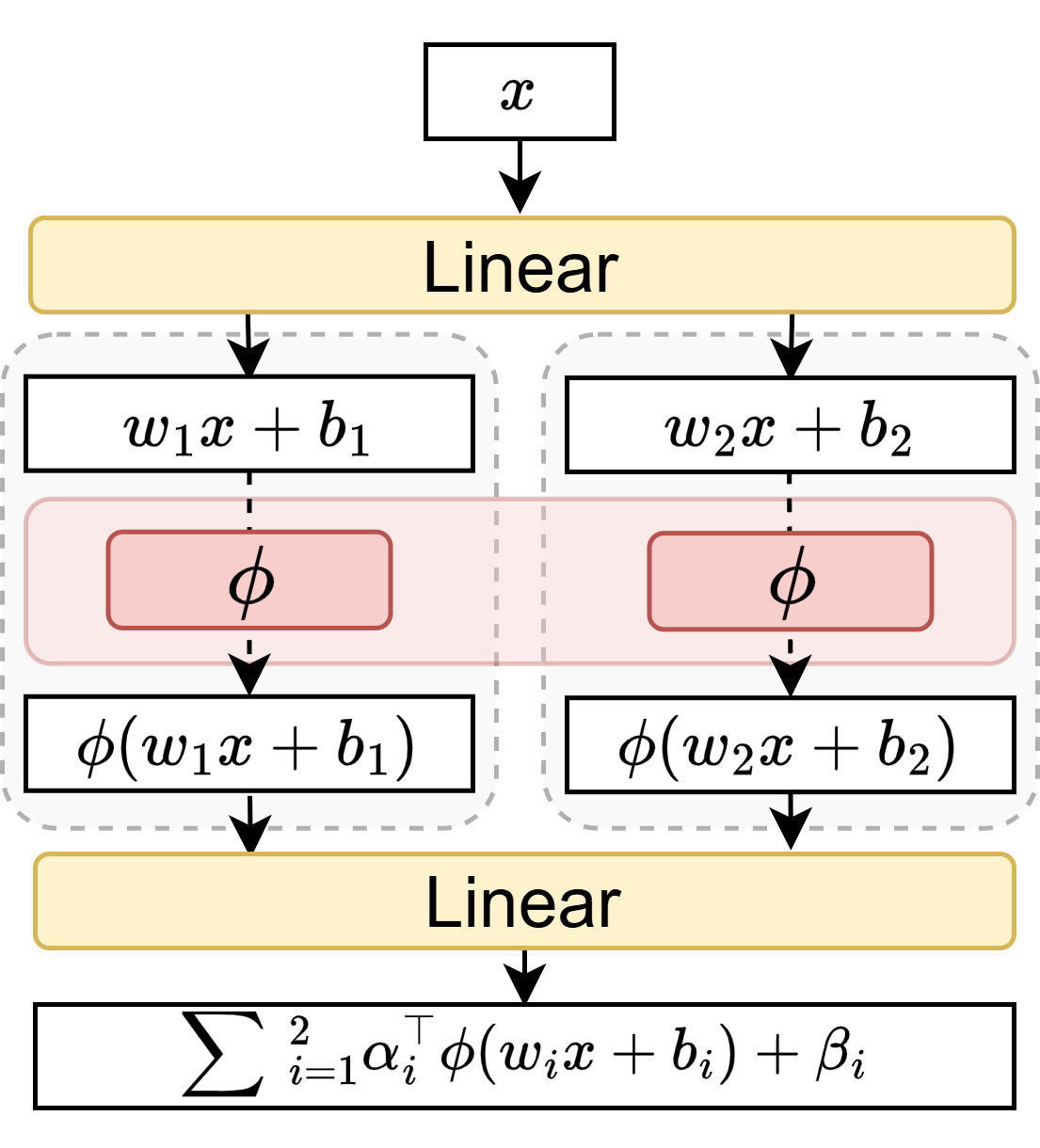}
    \captionof{figure}{$\phi$-Net decomposition
    (Corollary~\ref{coro:shallow PLN->phi}).}
    \label{fig:proof_pphi}
\end{minipage}
\par\medskip

\begin{theorem}[Universal approximation of univariate Lipschitz functions by shallow PLN-Nets]
\label{theorem:Lipschitz-UAT}
Given any $L$-Lipschitz continuous function $\hat f\in C([0,1])$ and every
$\epsilon>0$, there exists a shallow PLN-Net
$f\in\gF(\mathrm{PLN}[n_s];1,n_sN)$ with $n_s\ge2$ and
$N=\floor{L/(2\epsilon)}+1$ such that
\begin{equation}
|f-\hat{f}|_{L^\infty}< \epsilon.
\end{equation}
\end{theorem}

\begin{proofsketch}
The proof consists of three steps.  
Step 1: Construct $\bar f\in\gF(\mathrm{sign};1,N)$ such that $|\bar f-\hat f|_{L^{\infty}}<\epsilon$.  
Step 2: Use Lemmas~\ref{lemma:shallow LN->sign} and~\ref{lemma:PLN->LN} to represent $\bar f$ exactly by a PLN-Net $f$.  
Step 3: Combine the two bounds to obtain $|f-\hat f|_{L^\infty}<\epsilon$.
The complete proof is given in Appendix~\ref{sec:proof thm 3.2}.
\end{proofsketch}

\begin{remark}
The proof above uses a sign network as an intermediate representation. Since
Step~2 represents this network exactly instead of introducing another
approximation, the argument and its explicit width bound are particularly
direct. When $n_s\ge3$, a network with the smooth activation
$\phi(x)=x/\sqrt{1+x^2}$ can instead serve as the intermediate representation.
The same construction applies to practical LN with a stabilizing constant;
Appendix~\ref{sec:proof thm 3.2-2} gives the exact representation. The explicit
width bound in Theorem~\ref{theorem:Lipschitz-UAT} is specific to the sign-based
route.
\end{remark}

The same representation transfers standard multivariate approximation rates
for smooth non-polynomial activations to PLN-Nets:

\begin{theorem}[Multivariate approximation by shallow PLN-Nets]
\label{thm:shallow PLN->universal}
Let $d\ge2$, $s\ge1$, and let
$\hat f:B^d\to\sR$ belong to $\mathcal W^{s,\infty}(B^d)$ with
$\|\hat f\|_{\mathcal W^{s,\infty}}\le1$, where
$B^d=\{\rvx\in\sR^d:\|\rvx\|\le1\}$. For every $N\ge1$, there exists a
shallow PLN-Net $f\in\gF(\mathrm{PLN}[n_s];1,n_sN)$ with $n_s\ge3$ such that
\begin{equation}
    \|f-\hat f\|_{L^\infty(B^d)}\le C N^{-s/d},
\end{equation}
where $C=C(d,s)$ is independent of $N$.
\end{theorem}

\begin{proofsketch}
The smooth non-polynomial activation
$\phi(x)=x/\sqrt{1+x^2}$ admits the stated classical shallow-network rate
\citep{pinkus1999approximation}. Lemma~\ref{lemma:shallow LN->phi} and
Corollary~\ref{coro:shallow PLN->phi} represent the resulting width-$N$
$\phi$-network exactly by a width-$n_sN$ PLN-Net. The complete argument is in
Appendix~\ref{sec:proof thm 4.1}.
\end{proofsketch}

\paragraph{Extension to deep networks.}
The preceding results concern shallow networks. The same representation
mechanism extends layer by layer: each scalar activation in a deep activation
network is replaced by one fixed-dimensional LN block, and the blocks at the
same layer are arranged in parallel.

\begin{lemma}
    \label{lemma:deep PLN-Net to phi}
    Given any $\hat f\in\gF(\phi;L,N)$ with
    $\hat f:\sR^d\to\sR^m$, there exists a PLN-Net
    $f\in \gF(\mathrm{PLN}[n_s];L,n_sN)$ with $n_s\ge3$ such that
    $f=\hat f$. In short,
    \begin{equation}
        \gF(\phi;L,N)\subseteq\gF(\mathrm{PLN}[n_s];L,n_sN),
    \end{equation}
    where $\phi(x)=x/\sqrt{x^2+1}$.
\end{lemma}

Lemma~\ref{lemma:deep PLN-Net to phi} is the layer-wise extension of
Corollary~\ref{coro:shallow PLN->phi}. A deep $\phi$-Net is a composition of
shallow $\phi$ layers, and the corresponding shallow PLN constructions can be
composed in the same order. The detailed proof is given in
Appendix~\ref{sec:proof 4.1}.

This exact representation also allows approximation results for deep ReLU
networks to be transferred to PLN-Nets.

\begin{theorem}[Approximation of ReLU Networks by PLN-Nets]
\label{thm:deep PLN->ReLU}
Let $\hat f \in \gF(\mathrm{ReLU};L,N)$ with
$\hat f:[0,1]^d\to\sR^m$, and let $\epsilon>0$.
Then there exists a PLN-Net
$f \in \gF(\mathrm{PLN}[n_s];2L,3n_sN)$ with $n_s\ge3$ such that
$\|f-\hat f\|_{L^\infty([0,1]^d)}<\epsilon$.
\end{theorem}

Theorem~\ref{thm:deep PLN->ReLU} follows by combining a deep
$\phi$-network approximation of the target ReLU network
\citep{zhang2024deep} with Lemma~\ref{lemma:deep PLN-Net to phi}. The detailed
proof is provided in Appendix~\ref{sec:proof thm 4.2}.

\paragraph{Summary.}
LN already induces nontrivial scalar nonlinearities, but a single full-width LN
unit remains non-universal in shallow networks. PLN converts fixed-dimensional
normalization units into a width-scalable direct sum, allowing classical
activation-network approximation results to transfer to PLN-Nets in arbitrary
input dimension. The exact PLN--PLS equivalence established in
Lemma~\ref{lemma: PLN=PLS} transfers these conclusions directly to PLS-Nets.
We next lift the PLN constructions to structured architectures.

\section{Extensions to Structured Architectures}
\label{sec:4}
\label{sec:structured PLN}
\label{sec:structured approximation}

For sequence and image data, we identify PLN designs that preserve the
approximation guarantees established above.

\subsection{Sequence Models}
\label{sec:sequence models}

\begin{definition}[Sequence PLN]
\label{def:sequence PLN}
Let $\tX\in\sR^{B\times T\times C}$ be a sequence tensor, where $B$ is the
batch size, $T$ is the sequence length, and $C=Nn_s$ is the feature dimension.
For $G_r=\{(r-1)n_s+1,\ldots,rn_s\}$, sequence PLN acts token-wise as
$[\mathrm{PLN}^{\mathrm{seq}}_{n_s}(\tX)]_{b,t,G_r}
=\mathrm{LN}[n_s](\tX_{b,t,G_r})$.
\end{definition}

\paragraph{Position-wise affine layers.}
Suppressing the batch index, a position-wise affine layer shared across tokens
has the form $\varphi(\mX)=\mX\mW+\bm1_T\vb^\top$, where
$\mX\in\sR^{T\times d}$, $\mW\in\sR^{d\times m}$, and $\vb\in\sR^m$.

\begin{theorem}[Exact representation of position-wise FFNs by PLN]
\label{thm:transformer}
Let $\hat f=\hat\varphi_2\circ\phi\circ\hat\varphi_1$ be a position-wise
FFN of width $N$, where $\phi(x)=x/\sqrt{1+x^2}$ acts coordinate-wise. For
every $n_s\ge3$, there exist position-wise affine layers $\varphi_1$ and
$\varphi_2$ of intermediate width $n_sN$ such that
$f=\varphi_2\circ\mathrm{PLN}^{\mathrm{seq}}_{n_s}\circ\varphi_1$ equals
$\hat f$ on every sequence input.
\end{theorem}

Position-wise and ordinary affine layers both act on the last input dimension,
so the MLP construction applies at every token with shared parameters
(Appendix~\ref{sec:proof thm 4.6}). CNNs require a separate argument because
their kernels mix neighboring positions.

\subsection{Convolutional Neural Networks}
\label{sec:convolutional PLN}

For image tensors, Figure~\ref{fig:normalization-domains} contrasts PLN with
LN, conventional GN, and a mixed-axis alternative, clarifying why we
distinguish PLN from conventional GN.

\begin{figure}[H]
    \centering
    \includegraphics[width=\linewidth]{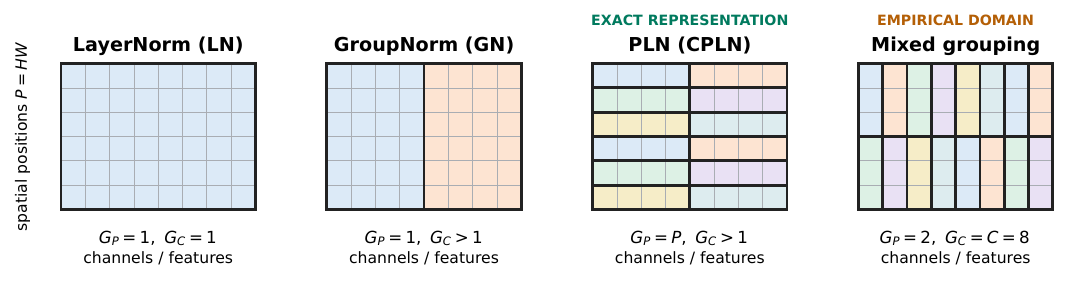}
    \caption{CNN normalization domains (schematic $P=HW=6$, $C=8$); boxes mark
    independent statistics. From left to right: (a) LN uses one $P\times C$
    domain; (b) GN groups channels across all positions~\citep{2018_ECCV_Wu};
    (c) CPLN groups channels at each position; and (d) mixed-axis grouping
    ($G_P=2,G_C=C$) is studied empirically.}
    \label{fig:normalization-domains}
\end{figure}

\begin{definition}[Image PLN]
\label{def:image PLN}
Let $\tX\in\sR^{B\times C\times H\times W}$ be an image tensor with
$C=Nn_s$. Image PLN acts independently at each spatial position as
$[\mathrm{PLN}^{\mathrm{img}}_{n_s}(\tX)]_{b,G_r,h,w}
=\mathrm{LN}[n_s](\tX_{b,G_r,h,w})$. We write
$\mathrm{CPLN}_{n_s}=\mathrm{PLN}^{\mathrm{img}}_{n_s}$.
\end{definition}

\begin{definition}[Affine Convolutional Layer]
\label{def:affine convolution}
Suppressing the batch index and using zero-based spatial indices, an affine
two-dimensional convolution with stride $(s_h,s_w)$, dilation $(d_h,d_w)$,
and zero padding $(p_h,p_w)$ has the form
\begin{equation}
\label{eqn:general convolution}
 [\mathcal C(\tX)]_{r,i,j}
 =
 b_r+\sum_{c=1}^{C_{\mathrm{in}}}\sum_{(u,v)\in S}
 K_{r,c,u,v}\,\widetilde{\tX}_{c,\,s_hi+d_hu-p_h,\,
 s_wj+d_wv-p_w}.
\end{equation}
Here $S\subset\sZ^2$ is the kernel support, and $\widetilde{\tX}$ is the
zero-padded extension of $\tX$. The output indices $(i,j)$ range over the grid
determined by these spatial hyperparameters.
\end{definition}

\begin{theorem}[Exact representation of convolutional blocks by CPLN]
\label{thm:cnn}
Let $\hat{\mathcal C}_1$ and $\hat{\mathcal C}_2$ be arbitrary affine
convolutional layers mapping channel dimensions $C_x\to N$ and
$N\to C_y$, respectively, and let
$\phi(t)=t/\sqrt{1+t^2}$ act element-wise. For every norm size $n_s\ge3$,
there exist affine convolutional layers $\mathcal C_1$ and $\mathcal C_2$
with the same kernel supports, strides, dilations, and padding as
$\hat{\mathcal C}_1$ and $\hat{\mathcal C}_2$, and intermediate width $n_sN$,
such that $\mathcal C_2\circ\mathrm{CPLN}_{n_s}\circ\mathcal C_1
=\hat{\mathcal C}_2\circ\phi\circ\hat{\mathcal C}_1$ for every input tensor.
\end{theorem}

The proof is in Appendix~\ref{sec:proof thm 4.7}. Layer-wise replacement and
folding the intervening $1\times1$ convolutions into adjacent layers gives an
exact deep CPLN network with unchanged depth and spatial specifications
(Appendix~\ref{sec:deep cnn representation}).

\paragraph{Summary.}
The MLP representation mechanism extends exactly to position-wise FFNs and
convolutional networks through shared affine maps and pointwise $1\times1$
convolutions. These results cover groups formed from feature coordinates
independently at each token or spatial position. Thus, the guarantee applies
to position-wise PLN, whereas normalization domains that mix tokens or spatial
positions are studied only empirically in Section~\ref{sec:experiments}.

\raggedbottom
\section{Experiments}
\label{sec:experiments}
\FloatBarrier

We first use MLP regression and physics-informed experiments to evaluate PLN's
approximation behavior. We then compare normalization domains in CNNs and ViTs,
including cross-spatial and cross-token groupings that lie outside the
preceding theory.

\subsection{Empirical Viability of PLN}
\label{sec:exp_mlp}
\label{sec:exp_function_approximation}

\paragraph{Function approximation.}
In a function-regression study with one-hidden-layer MLPs of width $1024$,
PLN[4] and PLN[8] reach mean test MSEs of $0.00483$ and $0.00793$,
respectively. PLN supplies the hidden-layer nonlinearity without an additional
element-wise activation. Full settings and results are in
Appendix~\ref{sec:dense fourier details}.

\paragraph{Physics-informed neural networks.}
\label{sec:exp_pinn}

\begin{wraptable}{r}{0.51\textwidth}
    \vspace{-1.2\baselineskip}
    \caption{Best relative $\mathcal{L}_2$ error for the Helmholtz PINN (lower is better). The lowest error at each width is shown in bold.}
    \label{tab:best_mse}
    \centering
    \scriptsize
    \setlength{\tabcolsep}{2.8pt}
    \begin{tabular}{r|cccc}
    \toprule
    Width & PLN[4] & SAT & Tanh & ReLU \\ \midrule
    128 & \textbf{0.000035} & 0.000104 & 0.000250 & 0.993593 \\
    256 & \textbf{0.000036} & 0.000052 & 0.000188 & 0.998485 \\ \bottomrule
    \end{tabular}
    \vspace{-0.6\baselineskip}
\end{wraptable}
We solve the one-dimensional Helmholtz equation with PINNs
\citep{raissi2019pinns}.
Table~\ref{tab:best_mse} reports the best relative $\mathcal L_2$ error at
widths $128$ and $256$; PLN[4] performs best at both, preliminarily indicating
its potential for PINNs. Full results are in Appendix~\ref{sec:experiment 2}.

\paragraph{Exploratory NTK diagnostics.}
We probe a possible spectral effect of grouping through the finite-model
output neural tangent kernel (NTK) \citep{jacot2018neural}, using a separate
Helmholtz PINN protocol (Appendix~\ref{sec:pinn_ntk_diagnostics}). PLN[4]
has a median initial 90\% effective rank of $12$, versus $1$--$2$ for the
other operators. Yet SAT attains the lowest median error after L-BFGS
(Table~\ref{tab:pinn_output_ntk}). Thus, higher initial output-kernel rank
alone neither predicts PINN error nor establishes an NTK mechanism for our
approximation results.

\subsection{Grouping and Normalization Domains in Image Models}
\label{sec:exp_cnn_grouping}
\label{sec:exp_normalization_domain}

Following \citet{2017_ICLR_Zhang}, we use fixed random labels to probe
finite-sample fitting capacity. We vary group count and grouped axes within a
fixed architecture. To assess performance on ordinary classification, we also
measure training and test accuracy with the original labels.

\paragraph{CIFAR-10 classification with CNNs.}
\begin{wraptable}{r}{0.56\textwidth}
    \vspace{-1.2\baselineskip}
    \caption{Representative CNN normalization domains on the $50{,}000$-example
    CIFAR-10 training set with fixed random labels (three-seed means).}
    \label{tab:cnn_operator_comparison}
    \centering
    \footnotesize
    \setlength{\tabcolsep}{2.5pt}
    \begin{tabular}{lcc}
    \toprule
    Operator & Normalization domain & Fitting accuracy (\%) \\ \midrule
    LN & $128\times32\times32$ & 24.6 \\
    Best GN & $2\times32\times32$ & 95.8 \\
    PLN[8] & $8\times1\times1$ & \textbf{100.0} \\
    PLN[4] & $4\times1\times1$ & \textbf{100.0} \\ \bottomrule
    \end{tabular}
    \vspace{-0.6\baselineskip}
\end{wraptable}
In a depth-$4$, width-$128$ CNN, grouped normalization replaces every nonlinear
layer. We split channels into $G_C$ groups and each spatial axis into contiguous
patches, so every normalization domain is one channel group within one patch.
Conventional GN uses the full spatial map, whereas position-wise PLN uses
single-position patches, as illustrated in Figure~\ref{fig:normalization-domains}(b)
and (c). Table~\ref{tab:cnn_operator_comparison} compares the empirical
expressive capacity of LN, GN, and PLN through random-label fitting.
Under the same CNN
architecture, PLN fits all labels and outperforms GN, which in turn outperforms
LN. The full mixed-axis grid is in Appendix~\ref{sec:cnn grouping details}.

In a separate VGG-16 study, we compare PLN[8] with other activation functions
across learning rates and initialization scales. PLN[8] trains reliably over a
broad range of these settings; detailed results are in
Appendix~\ref{sec:experiment 3}.

\paragraph{CIFAR-10 classification with ViTs.}
\label{sec:exp_sequence_grouping}
We evaluate ViTs with the same fixed-random-label and original-label tasks.
Grouped normalization is applied at two sites: \emph{all-normalization
replacement} replaces LN and retains GELU, while \emph{activation replacement}
retains LN and replaces FFN GELUs. For $\mathbb{R}^{B\times T\times C}$, we
form $G_TG_C$ token--channel blocks; $G_T=T$ is PLN and $G_T<T$ is Mix-Norm.
Complete ViT configuration grids and protocols are in
Appendix~\ref{sec:vit grouping details}.

\begin{wraptable}{r}{0.58\textwidth}
    \vspace{-0.9\baselineskip}
    \caption{ViT accuracy (\%). Baseline is the unchanged LN--GELU model;
    arrows indicate replacement of all LN layers or all FFN GELUs. Each cell
    reports the highest three-seed mean / highest single-run accuracy across
    configurations in that row, selected independently.}
    \label{tab:vit_grouping_summary}
    \centering
    \footnotesize
    \setlength{\tabcolsep}{1.5pt}
    \begin{tabular}{@{}lccc@{}}
    \toprule
    Setting & Random 20k & Clean train & Clean test \\ \midrule
    Baseline             & 48.2 / 72.8 & 100.0 / 100.0 & 71.6 / 72.3 \\
    LN$\to$PLN            & 75.3 / 92.1 & 100.0 / 100.0 & 72.3 / 72.9 \\
    LN$\to$Mix-Norm       & 89.5 / 91.2 & 100.0 / 100.0 & \textbf{73.5 / 74.5} \\
    GELU$\to$PLN          & 73.7 / 97.3 & 99.4 / 100.0  & 69.7 / 71.0 \\
    GELU$\to$Mix-Norm     & \textbf{91.7 / 98.2} & 100.0 / 100.0 & 70.2 / 70.7 \\
    \bottomrule
    \end{tabular}
    \vspace{-0.6\baselineskip}
\end{wraptable}
Table~\ref{tab:vit_grouping_summary} shows that, on $20{,}000$ samples
with random labels, PLN fits better than the LN--GELU baseline at both
replacement sites. Mix-Norm achieves the highest three-seed mean at both
sites; unlike position-wise PLN, it groups across tokens and therefore lies
outside the setting of our approximation theory. These comparisons show
stronger random-label fitting with grouped normalization in these finite ViTs.

Interestingly, the PLN--LN contrast changes markedly between our random-label
fitting and clean-label classification experiments. Compared with LN, PLN shows a clear advantage in fitting random labels,
but its advantage in clean-label classification is limited. Its clean-test
accuracy is only slightly higher under normalization replacement ($72.3\%$
versus $71.6\%$) and lower under activation replacement ($69.7\%$ versus
$71.6\%$). Thus, greater random-label fitting capacity does not translate
into a comparable gain in clean-label generalization.

\paragraph{Summary.}
The shallow MLP experiments provide finite-network evidence consistent with
PLN's universal-approximation results, while the PINN study suggests its
potential for physics-informed learning. Across CNNs and ViTs, comparisons
among LN, PLN, GN, and Mix-Norm show that performance depends not only on group
count but also on the grouped axes, replacement site, and evaluation task
(random-label fitting versus clean-label generalization).

\FloatBarrier
\section{Conclusion}
\label{sec:conclusion}

This work identifies a distinction between two ways to scale
normalization-based networks: widening one LN unit does not make a shallow
LN-Net universal, whereas adding parallel fixed-dimensional units does. By
formalizing the latter as PLN, we establish shallow universal-approximation
results with explicit bounds and transfer approximation constructions to deep
PLN-Nets. The exact PLN--PLS equivalence extends these results to RMSNorm-based
networks. Besides, our representation results carry the local mechanism to
position-wise feed-forward and convolutional networks. Finite-model experiments
support PLN's approximation behavior and show that normalization-domain choices
affect CNN and ViT performance.

\paragraph{Limitations and Future Work.}
Our theorems cover affine--PLN compositions with normalization local to a
token or spatial position, but not cross-position domains or arbitrary
compositions with additional nonlinear activations. The experiments provide
finite-model evidence rather than empirical proof of universality, and the PINN
results remain preliminary. In future work, we will extend the theory and
experiments to broader tasks and cross-position normalization.

\section*{Acknowledgments}
This work was supported by the National Natural Science Foundation of China
(Grant Nos. 62476016 and 62441617), the Fundamental Research Funds for the
Central Universities, and the Beijing Natural Science Foundation
(Grant No. QY24137).

\section*{Code Availability}
Experiment code and selected numerical records are available at
\url{https://github.com/Musicath/Parallel-LN}. The repository documents the
scope and limitations of the released artifacts.

\section*{AI Use Statement}
We used generative AI tools to assist with manuscript editing, proof checking,
literature search and citation checking, code implementation and debugging,
and the preparation and interpretation of synthetic experiments and figures.
The research question, mathematical claims and constructions, and experimental
objectives were formulated by the authors. All AI-assisted outputs were
reviewed and verified by the authors, who take full responsibility for the
final content.



\bibliography{main}
\bibliographystyle{iclr2027_conference}
\newpage

\appendix
\clearpage
\renewcommand{\thetable}{\Roman{table}}
\setcounter{table}{0}

\renewcommand{\thefigure}{A\arabic{figure}}
\setcounter{figure}{0}
\everymath{\displaystyle}

\section{Mathematical Proofs in Section~\ref{sec:extension}}

\subsection{Proof of Lemma \ref{lemma:shallow LN->sign}}
\label{sec:proof 3.1}

\paragraph{Lemma \ref{lemma:shallow LN->sign}}
Given any $\hat{f}\in\gF(\mathrm{sign};1,1)$ with $\hat f:\sR\to\sR$, there exists $f\in\gF(\mathrm{LN}[n_s];1,n_s)$ with $n_s\ge2$ such that $f=\hat f$. In short,
\begin{equation}
\gF(\mathrm{sign};1,1)\subseteq\gF(\mathrm{LN}[n_s];1,n_s),
\end{equation}
where $\mathrm{sign}(x)$ is the element-wise sign function taking values $1$ if $x>0$, $0$ if $x=0$, and $-1$ if $x<0$.

\begin{proof}
    
    For any  $\hat{f}\in\gF(\mathrm{sign};1,1)$ and $\hat f :\sR\to\sR$, we have $\hat f(x)=w_2\sign(w_1x+b_1)+b_2$. 
    
    For $f\in\gF({\mathrm{LN}[n_s]};1,n_s)$ with $n_s\ge2$, we construct $f(x)=\vv_2^\top\mathrm{LN}(\vv_1x+\vc_1)+c_2$, where $\vv_1,\vc_1,\vv_2\in\sR^{n_s}$ and $c_2\in\sR$. 
    
    Let
    \begin{equation}
        \vv_1=\mat{w_1\\-w_1\\\bm0_{n_s-2}},\vc_1=\mat{b_1\\-b_1\\\bm0_{n_s-2}}, \vv_2=\mat{\sqrt{2/n_s}w_2\\0\\\bm0_{n_s-2}},c_2=b_2.
    \end{equation}
    We can identify that the first element of $\mathrm{LN}(\vv_1x+\vc_1)$ is
    \begin{equation}
        \frac{w_1x+b_1}{\sqrt{2(w_1x+b_1)^2/n_s}}=\sqrt\frac{n_s}{2}\frac{w_1x+b_1}{|w_1x+b_1|}=\sqrt\frac{n_s}{2}\sign(w_1x+b_1),
    \end{equation}
    if $w_1x+b_1\ne0$. 

    For the case $w_1x+b_1=0$, we have $\sigma^2=2(w_1x+b_1)^2/n_s=0$, then $\mathrm{LN}(\vv_1x+\vc_1)=\bm0$ by its definition. 

    Therefore, the first element of $\mathrm{LN}(\vv_1x+\vc_1)$ is equal to $\sqrt\frac{n_s}{2}\sign(w_1x+b_1)$, thus
    \begin{equation}
    \begin{aligned}
        f(x)&=\vv_2^\top\mathrm{LN}(\vv_1x+\vc_1)+c_2 \\
        &=\sqrt{2/n_s}w_2\cdot \sqrt\frac{n_s}{2}\sign(w_1x+b_1)+c_2 \\
        &= w_2\sign(w_1x+b_1)+b_2=\hat f(x).
    \end{aligned}
    \end{equation}
    This completes the proof. 
\end{proof}

We can extend Lemma \ref{lemma:shallow LN->sign} to the version $\hat f:\sR^d\to\sR^m$. 

\begin{corollary}
\label{corollary:A1}
    Given any $\hat{f}\in\gF(\mathrm{sign};1,1)$ with $\hat f:\sR^d\to\sR^m$, there exists $f\in\gF({\mathrm{LN}[n_s]};1,n_s)$ with $n_s\ge2$, such that $f=\hat f$. In short,
    \begin{equation}
    \gF(\mathrm{sign};1,1)\subseteq\gF({\mathrm{LN}[n_s]};1,n_s),
    \end{equation}
    where $\mathrm{sign}(x)$ is the element-wise sign function, which takes $1$ if $x>0$, $0$ if $x=0$, and $-1$ if $x<0$. 
\end{corollary}

\begin{proof}
    For any $\hat{f}\in\gF(\mathrm{sign};1,1)$ with
    $\hat f :\sR^d\to\sR^m$, write
    $\hat f(\rvx)=\vw_2\sign(\vw_1^\top\rvx+b_1)+\vb_2$, where
    \begin{equation}
        \vw_1\in\sR^d,b_1\in\sR,\vw_2,\vb_2\in\sR^m. 
    \end{equation}
    Define
    \begin{equation}
        \mV_1=\mat{\vw_1^\top\\-\vw_1^\top\\\mO_{(n_s-2)\times d}}
        \in\sR^{n_s\times d},\qquad
        \vc_1=\mat{b_1\\-b_1\\\bm0_{n_s-2}}\in\sR^{n_s},
    \end{equation}
    and
    \begin{equation}
        \mV_2=\mat{\sqrt{2/n_s}\,\vw_2 & \mO_{m\times(n_s-1)}}
        \in\sR^{m\times n_s}.
    \end{equation}
    Then
    $f(\rvx)=\mV_2\mathrm{LN}(\mV_1\rvx+\vc_1)+\vb_2=\hat f(\rvx)$.
    
    The following proof is almost the same as that in the proof of Lemma \ref{lemma:shallow LN->sign}. 
\end{proof}

\subsection{Proof of Lemma \ref{lemma:PLN->LN}}
\label{sec:proof 3.3}

\paragraph{Lemma \ref{lemma:PLN->LN}}   
    Given $N$ shallow LN-Nets
    $\hat f_i\in\gF(\mathrm{LN}[n_s];1,n_s)$ with
    $\hat f_i:\sR^{d_x}\to\sR^{d_y}$ for $1\le i\le N$, there exists a
    shallow PLN-Net $f\in\gF(\mathrm{PLN}[n_s];1,n_sN)$ such that
    \begin{equation}
    f(\rvx) = \sum_{i=1}^N \hat f_i(\rvx),\qquad
    \forall \rvx\in\sR^{d_x}.
    \end{equation}
    Conversely, any such PLN-Net can be decomposed into a sum of $N$ shallow
    LN-Nets with the same input and output dimensions.
\begin{proof}
    Given $N$ shallow LN-Nets $\hat f_i\in\gF(\mathrm{LN}[n_s];1,n_s)$ for
    $1\le i\le N$, write
    $\hat f_i(\rvx) = \mW_{i2}\mathrm{LN}(\mW_{i1}\rvx+\vb_{i1})+\vb_{i2}$,
    where
    \begin{equation}
        \mW_{i1}\in\sR^{n_s\times d_x},\quad
        \vb_{i1}\in\sR^{n_s},\quad
        \mW_{i2}\in\sR^{d_y\times n_s},\quad
        \vb_{i2}\in\sR^{d_y}.
    \end{equation}

    Define $f \in\gF(\mathrm{PLN}[n_s];1,n_sN)$ by
    $f(\rvx) = \mW_{2}\mathrm{PLN}(\mW_{1}\rvx+\vb_{1})+\vb_{2}$, where
    \begin{equation}
    \mW_1\in\sR^{n_sN\times d_x},\quad
    \vb_1\in\sR^{n_s N},\quad
    \mW_2\in\sR^{d_y\times n_sN},\quad
    \vb_2\in\sR^{d_y}.
    \end{equation}

    Set
    \begin{equation}
        \mW_1 =\mat{\mW_{11}\\\mW_{21}\\\vdots\\\mW_{N1}}, \vb_1 =\mat{\vb_{11}\\\vb_{21}\\\vdots\\\vb_{N1}}, \mW_2 =\mat{\mW_{12}&\mW_{22}&\cdots&\mW_{N2}}, \vb_{2}=\sum_{i=1}^N \vb_{i2}. 
    \end{equation}
    By blockwise application of PLN, this construction gives
    $f(\rvx)=\sum_{i=1}^N\hat f_i(\rvx)$. Conversely, split the input
    matrix, hidden bias, and output matrix of any PLN-Net into $N$ conforming
    blocks and distribute its output bias across the $N$ LN-Nets. This yields
    the stated decomposition.
\end{proof}

\subsection{Proof of Lemma \ref{lemma: PLN=PLS}}
\label{sec:proof 3.4}

Lemma \ref{lemma: PLN=PLS} can simplify most the proofs, so we prove it at the early time. 

\paragraph{Lemma \ref{lemma: PLN=PLS}}
    Shallow LN-Nets with width $n_s$ and shallow LS-Nets with width $n_s-1$ can represent each other, namely
    \begin{equation}
        \gF(\mathrm{LN}[n_s];1,n_s) = \gF(\mathrm{LS}[n_s-1] ; 1, n_s-1), 
    \end{equation}
    for $n_s\ge2$. Furthermore, 
    \begin{equation}
        \gF(\mathrm{PLN}[n_s];1,n_sN) = \gF(\mathrm{PLS}[n_s-1] ; 1, (n_s-1)N).
    \end{equation}

We first require Lemma~\ref{lemma:LN<->LS}, which follows from Lemma~6 of
\citet{ni2024nonlinearity}.

\begin{lemma}[\citealp{ni2024nonlinearity}, Lemma~6]
    \label{lemma:nl-of-ln}
		There is some orthogonal matrix $ \mQ \in \sR^{d \times d} $, such that $ \vz = \mQ\mat{\rvx\\0} \in \{\vz \in \sR^d:z^{(1)} + \cdots + z^{(d)} = 0\} $ (namely $ \vz $ is centralized), for $ \rvx \in \sR^{d-1} $,
\end{lemma}

\begin{lemma}
    \label{lemma:LN<->LS}
    Denote $ \mathrm{LN}(\cdot) $ as the LN operation in $ \sR^d (d\ge3) $, and $ \mathrm{LS}(\cdot) $ as the LS operation in $ \sR^{d-1} $. We can find some linear transformations $ \varphi_1 $ and $ \varphi_2 $, such that 
		\begin{equation}
			\mathrm{LS}(\cdot) = \varphi_2 \circ \mathrm{LN}(\cdot) \circ \varphi_1. 
		\end{equation} 
    Besides, we can find some linear transformations $ \varphi_1^* $ and $ \varphi_2^* $, such that
        \begin{equation}
			\mathrm{LN}(\cdot) = \varphi_2^* \circ \mathrm{LS}(\cdot) \circ \varphi_1^*. 
		\end{equation} 
\end{lemma}

\begin{proof}
    For $\rvx\in\sR^{d-1}$ and $\rvy\in\sR^d$,
    Equation~\ref{eqn:normalization operators} gives
    \begin{equation}
        \mathrm{LS}(\rvx)=\sqrt{d-1}\,\rvx/\|\rvx\|_2,
        \qquad
        \mathrm{LN}(\rvy)=\sqrt d\,\mP_d\rvy/\|\mP_d\rvy\|_2.
    \end{equation}

    Based on the orthogonal matrix $\mQ$, we obtain
        \begin{equation}
            \|\vz\|_2 = \left\Vert \mQ \mat{\rvx \\ 0} \right\Vert_2 = \left\Vert \mat{\rvx \\ 0} \right\Vert_2 = \Vert \rvx \Vert_2 .
        \end{equation}
    By Lemma \ref{lemma:nl-of-ln}, we have $\bm1^\top\vz=0$ and
    $\mP_d\vz=\vz$. We hence find that
    \begin{equation}
            \mathrm{LN}(\vz)=\sqrt d\,\vz/\|\vz\|_2 .
    \end{equation}
    Let $ \mI_{d} $ denotes the identity matrix in $ \sR^{d\times d} $. We thus have
    \begin{equation}
        \begin{aligned}
            \frac{\sqrt{d-1}}{\sqrt{d}}\mat{\mI_{d-1}&\bm0}\mQ^\top \mathrm{LN}(\mQ\mat{\mI_{d-1} & \bm0}^\top \rvx) 
            =&\frac{\sqrt{d-1}}{\sqrt{d}}\mat{\mI_{d-1}&\bm0}\mQ^\top \mathrm{LN}(\vz) \\
            =&\frac{\sqrt{d-1}}{\sqrt{d}}\mat{\mI_{d-1}&\bm0}\mQ^\top \sqrt{d}\ \vz/\|\vz\|_2 \\
            =&\frac{\sqrt{d-1}}{\sqrt{d}}\sqrt{d}\mat{\mI_{d-1}&\bm0} \mQ^\top \mQ \mat{\rvx\\0} / \|\rvx\|_2 \\
            =&\sqrt{d-1}\ \rvx / \|\rvx\|_2 \\
            =&\mathrm{LS}(\rvx).
        \end{aligned}
    \end{equation}
    Let $\varphi_1(\rvx) = \mQ \mat{\mI_{d-1} & \bm0}^\top \rvx $ and $\varphi_2(\rvx) = \frac{\sqrt{d-1}}{\sqrt{d}}\mat{\mI_{d-1} & \bm0} \mQ^\top \rvx $. We observe that
	\begin{equation}
		\mathrm{LS}(\cdot) = \varphi_2 \circ \mathrm{LN}(\cdot) \circ \varphi_1. 
	\end{equation}

    
    Now we prove the inverse equation. According to $\rvz = \mQ\mat{\mI_{d-1} & \bm0}^\top\rvx$, we have  
    \begin{equation}
        \begin{aligned}
            \mathrm{LS}(\mat{\mI_{d-1} & \bm0}\mQ^\top \rvz) 
            &= \mathrm{LS}(\mat{\mI_{d-1} & \bm0}\mQ^\top \mQ \mat{\mI_{d-1} & \bm0}^\top \rvx)\\
            &=\mathrm{LS}(\mat{\mI_{d-1} & \bm0}\mat{\mI_{d-1} & \bm0}^\top \rvx) \\
            &= \mathrm{LS}(\mI_{d-1}\ \rvx)\\
            &=\mathrm{LS}(\rvx).
        \end{aligned}
    \end{equation}
    We thus have
    \begin{equation}
        \begin{aligned}
        \frac{\sqrt{d}}{\sqrt{d-1}}\ \mQ \mat{\mI_{d-1} & \bm0}^\top \mathrm{LS}( \mat{\mI_{d-1} & \bm0} \mQ^\top \rvz ) 
        &=  \frac{\sqrt{d}}{\sqrt{d-1}}\ \mQ \mat{\mI_{d-1} & \bm0}^\top \mathrm{LS}(\rvx)\\
        &= \frac{\sqrt{d}}{\sqrt{d-1}}\ \mQ \mat {\mI_{d-1} & \bm0}^\top\sqrt{d-1}\ \rvx / \|\rvx\|_2 \\
        &=  \sqrt{d}\ \rvz / \|\rvz\|_2 \\
         &= \mathrm{LN}(\rvz).
        \end{aligned}
    \end{equation}
    The calculation above covers every centered input $\vz$. For an arbitrary
    $\rvy\in\sR^d$, let
    $\vz=\mP_d\rvy$. The first $d-1$ columns of
    $\mQ$ span the centered subspace, so
    $\mat{\mI_{d-1}&\bm0}\mQ^\top\rvy
    =\mat{\mI_{d-1}&\bm0}\mQ^\top\rvz$, while
    $\mathrm{LN}(\rvy)=\mathrm{LN}(\rvz)$. This also holds at zero variance
    under the stated zero-output conventions.

    Let $\varphi_1^*(\rvx) = \mat{\mI_{d-1} & \bm0} \mQ^\top \rvx $ and
    $\varphi_2^*(\rvx) =  \frac{\sqrt{d}}{\sqrt{d-1}}\ \mQ
    \mat{\mI_{d-1} & \bm0}^\top \rvx$.
    We observe that
	\begin{equation}
		\mathrm{LN}(\cdot) = \varphi_2^* \circ \mathrm{LS}(\cdot) \circ \varphi_1^*. 
	\end{equation}
 

\end{proof}

For $d=2$, the same two identities follow directly without
Lemma~\ref{lemma:nl-of-ln}. Under the conventions
$\mathrm{LS}(0)=0$ and $\mathrm{LN}(c\bm1)=\bm0$, define
\begin{align}
    \varphi_1(t)&=(t,0)^\top,
    &\varphi_2(z_1,z_2)&=z_1,\\
    \varphi_1^*(z_1,z_2)&=z_1-z_2,
    &\varphi_2^*(t)&=(t,-t)^\top.
\end{align}
Since $\mathrm{LS}(t)=\sign(t)$ and
$\mathrm{LN}(z_1,z_2)=\bigl(\sign(z_1-z_2),-\sign(z_1-z_2)\bigr)^\top$,
these maps satisfy
\begin{equation}
    \mathrm{LS}=\varphi_2\circ\mathrm{LN}\circ\varphi_1,
    \qquad
    \mathrm{LN}=\varphi_2^*\circ\mathrm{LS}\circ\varphi_1^*.
\end{equation}

Now we prove Lemma \ref{lemma: PLN=PLS}.
\begin{proof}
    Step 1: Represent LN-Nets by LS-Nets. Namely given any $\hat f\in\gF(\mathrm{LN}[n_s];1,n_s)$, there exists $f \in \gF(\mathrm{LS}[n_s-1] ; 1, n_s-1)$, such that $f = \hat f$. 

    Set $\hat f = \mW_2\circ\mathrm{LN}\circ\mW_1$, where
    $\mW_1:\sR^d\to\sR^{n_s}$ and
    $\mW_2:\sR^{n_s}\to\sR^m$ are affine layers. For $n_s\ge3$, use
    Lemma~\ref{lemma:LN<->LS}; for $n_s=2$, use the explicit maps above. In
    either case,
    $\mathrm{LN}=\varphi_2^*\circ\mathrm{LS}\circ\varphi_1^*$.
    Hence choose $\mV_1=\varphi_1^*\circ\mW_1$ and
    $\mV_2=\mW_2\circ\varphi_2^*$ to obtain
    $f=\mV_2\circ\mathrm{LS}\circ\mV_1=\hat f$.

    Step 2: Represent LS-Nets by LN-Nets. Namely given any $\hat f \in \gF(\mathrm{LS}[n_s-1] ; 1, n_s-1)$, there exists $f\in\gF(\mathrm{LN}[n_s];1,n_s)$, such that $f = \hat f$. 
    
    Set $\hat f = \mW_2\circ\mathrm{LS}\circ\mW_1$, where
    $\mW_1:\sR^d\to\sR^{n_s-1}$ and
    $\mW_2:\sR^{n_s-1}\to\sR^m$ are affine layers. Using
    $\mathrm{LS}=\varphi_2\circ\mathrm{LN}\circ\varphi_1$, choose
    $\mV_1=\varphi_1\circ\mW_1$ and
    $\mV_2=\mW_2\circ\varphi_2$. Then
    $f=\mV_2\circ\mathrm{LN}\circ\mV_1=\hat f$.

    Step 3: Extend to parallel structures. 

    For simplicity, we define the sum of networks as
    \begin{equation}
        \gS(m,\phi;L,N) = \left\{f=\sum_{i=1}^m f_i:f_i\in\mathcal{F}(\phi;L,N),1\le i\le m\right\}, 
    \end{equation}
    which represents the sum of $m$ functions from $\mathcal{F}(\phi;L,N)$. 

    Then by Lemma \ref{lemma:PLN->LN}, we obtain that
    \begin{equation}
        \gS(N,\mathrm{LN}[n_s];1,n_s) = \gF(\mathrm{PLN}[n_s];1,n_sN). 
    \end{equation}

    With the similar proof of Lemma \ref{lemma:PLN->LN}, we will obtain
    \begin{equation}
        \gS(N,\mathrm{LS}[n_s-1];1,n_s-1) = \gF(\mathrm{PLS}[n_s-1];1,(n_s-1)N).
    \end{equation} 

    Since $\gF(\mathrm{LN}[n_s];1,n_s) = \gF(\mathrm{LS}[n_s-1] ; 1, n_s-1)$, it is easy to identify that 
    \begin{equation}
        \gS(N,\mathrm{LS}[n_s-1];1,n_s-1) = \gS(N,\mathrm{LN}[n_s];1,n_s), 
    \end{equation}
    by the definition of $\mathcal{S}$. 

    Therefore, we obtain that 
    \begin{equation}
        \gF(\mathrm{PLN}[n_s];1,n_sN) = \gF(\mathrm{PLS}[n_s-1] ; 1, (n_s-1)N). 
    \end{equation}
    Here we complete the proof. 
\end{proof}

\subsection{Proof of Lemma \ref{lemma:shallow LN->phi}}
\label{sec:proof 3.2}

\paragraph{Lemma \ref{lemma:shallow LN->phi}}

    Given any $\hat{f}\in\gF(\phi;1,1)$ with $\hat f:\sR\to\sR$, there exists $f\in\gF(\mathrm{LN}[n_s];1,n_s)$ with $n_s\ge3$ such that $f=\hat f$. In short,
    \begin{equation}
    \gF(\phi;1,1)\subseteq \gF(\mathrm{LN}[n_s];1,n_s),
    \end{equation}
    where $\phi(x)=x/\sqrt{x^2+1}$.

\begin{proof}
    To simplify the proof, we apply Lemma \ref{lemma: PLN=PLS} first, which says that $\gF({\mathrm{LN}[n_s]};1,n_s)=\gF({\mathrm{LS}[n_s-1]};1,n_s-1)$

    For any $\hat{f}\in\gF(\phi;1,1)$, we have $\hat f(x)=w_2\phi(w_1x+b_1)+b_2$, 
    where $\phi(x)=x/\sqrt{x^2+1}$. 
    
    For $f\in\gF({\mathrm{LS}[n_s-1]};1,n_s-1)$ with $n_s\ge3$, construct 
    $f(x)=\vv_2^\top\mathrm{LS}(\vv_1x+\vc_1)+c_2$, where $\vv_1,\vv_2,\vc_1\in\sR^{n_s-1}$ and $c_2\in\sR$. 
    
    Let
    \begin{equation}
        \vv_1=\mat{w_1\\0\\\bm0_{n_s-3}},\vc_1=\mat{b_1\\1\\\bm0_{n_s-3}}, \vv_2=\mat{w_2/\sqrt{n_s-1}\\0\\\bm0_{n_s-3}},c_2=b_2.
    \end{equation}
    We can identify that the first element of $\mathrm{LS}(\vv_1x+\vc_1)$ is
    \begin{equation}
        \frac{w_1x+b_1}{\sqrt{[(w_1x+b_1)^2+1]/(n_s-1)}}
        =\sqrt{n_s-1}\frac{w_1x+b_1}{\sqrt{(w_1x+b_1)^2+1}}
        =\sqrt{n_s-1}\phi(w_1x+b_1).
    \end{equation}

    Therefore, we obtain
    \begin{equation}
    \begin{aligned}
        f(x)&=\vv_2^\top\mathrm{LS}(\vv_1x+\vc_1)+c_2 \\
        &=w_2/\sqrt{n
        _s-1}\cdot \sqrt{n_s-1}\phi(w_1x+b_1) + c_2 \\
        &= w_2\phi(w_1x+b_1)+b_2=\hat f(x).
    \end{aligned}
    \end{equation}
    This completes the proof.
\end{proof}

We can also extend it to high dimensions.

\begin{corollary}
\label{corollary:A3}
    Given any $\hat{f}\in\gF(\phi;1,1)$ and $\hat f:\sR^d\to\sR^m$, there exists some $f\in\gF({\mathrm{LN}[n_s]};1,n_s)$ with $n_s\ge3$, such that $f=\hat f$. In other words, we have
    \begin{equation}
        \gF(\phi;1,1)\subseteq \gF({\mathrm{LN}[n_s]};1,n_s), 
    \end{equation}
    where $\phi(x)=x/\sqrt{x^2+1}$ is an activation function.
\end{corollary}

\begin{proofsketch}
    Follow the scalar construction above, replacing $w_1x$ by
    $\vw_1^\top\rvx$ and the scalar output weight by a vector in $\sR^m$.
\end{proofsketch}

\subsection{Proof of Theorem \ref{thm:LN-Nets weak}}
\label{sec:proof thm 3.1}

We first characterize the possible shapes of a scalar function represented by
one shallow LN unit.

\begin{lemma}[Structure of scalar shallow LN-Nets]
\label{lemma:stationary point}
Let $n_s\ge2$ and let
$f\in\gF(\mathrm{LN}[n_s];1,n_s)$ map $\sR$ to $\sR$. Then either
$f\in\gF(\sign;1,1)$, or $f\in C^\infty(\sR)$ is nonconstant and $f'$ has
at most one zero.
\end{lemma}

\begin{proof}
Let $d=n_s-1$. By Lemma~\ref{lemma: PLN=PLS}, every such function can be
written as
\begin{equation}
\label{eqn:scalar shallow LS form}
 f(x)=\frac{\vv^\top(\vw x+\vb)}{\|\vw x+\vb\|_2}+b_0,
\end{equation}
where $\vw,\vb,\vv\in\sR^d$ and $b_0\in\sR$; the factor $\sqrt d$ in
$\mathrm{LS}(\rvh)=\sqrt d\,\rvh/\|\rvh\|_2$ has been absorbed into the
free output vector $\vv$.

If $\vw=\bm0$, then $f$ is constant and hence belongs to
$\gF(\sign;1,1)$. If $\vw\ne\bm0$ and $\vb=k\vw$ for some $k\in\sR$,
then, using the zero-output convention at $x=-k$,
\begin{equation}
 f(x)=\frac{\vv^\top\vw}{\|\vw\|_2}\sign(x+k)+b_0,
\end{equation}
so again $f\in\gF(\sign;1,1)$.

It remains to consider $\vw\ne\bm0$ and $\vb\notin\operatorname{span}\{\vw\}$.
Then $\vw x+\vb\ne\bm0$ for every $x$, and strict Cauchy--Schwarz gives
\begin{equation}
 \inf_{x\in\sR}\|\vw x+\vb\|_2^2
 =\|\vb\|_2^2-\frac{(\vw^\top\vb)^2}{\|\vw\|_2^2}>0.
\end{equation}
Thus $f\in C^\infty(\sR)$. Set
$\alpha=\vv^\top\vw$, $\beta=\vv^\top\vb$,
$p=\|\vw\|_2^2$, $q=\vw^\top\vb$, and $r=\|\vb\|_2^2$. Direct
differentiation yields
\begin{equation}
 f'(x)=
 \frac{(\alpha q-\beta p)x+(\alpha r-\beta q)}
 {\|\vw x+\vb\|_2^3}.
\end{equation}
If the two numerator coefficients are not both zero, this derivative has at
most one zero. Otherwise,
$\alpha q=\beta p$ and $\alpha r=\beta q$. Since
$pr-q^2>0$, these identities imply $\alpha=\beta=0$, and hence
$f\equiv b_0$ is constant. Such a function belongs to
$\gF(\sign;1,1)$, completing the dichotomy.
\end{proof}

Now we prove Theorem \ref{thm:LN-Nets weak} based on Lemma \ref{lemma:stationary point}. 

\paragraph{Theorem \ref{thm:LN-Nets weak}}(Weak representation capacity of LN-Nets). 
There exists a continuous function $\hat f$ on $\Omega\subset \sR$, such that for any shallow LN-Net $f\in\gF(\mathrm{LN}[n_s];1,n_s)$ with $n_s\ge 2$,
\begin{equation}
|f-\hat{f}|_{L^\infty(\Omega)}\ge 1.
\end{equation}

\begin{proof}
Take $\Omega=[-2,2]$ and $\hat f(x)=\cos(\pi x)$. Suppose for contradiction
that $|f(x)-\hat f(x)|<1$ on $[-2,2]$.

First suppose $f\in\gF(\sign;1,1)$ and write
$f(x)=w_2\sign(w_1x+b_1)+b_2$. If $w_1=0$, then $f$ is constant, so
$f(1)=f(2)$. If $w_1\ne0$, let $t_0=-b_1/w_1$. When $t_0\le0$, the
function has the same value at $1$ and $2$; when $t_0>0$, it has the same
value at $-2$ and $-1$. Thus, in every case, there are $x_1,x_2$ equal to
either $1,2$ or $-2,-1$ such that $f(x_1)=f(x_2)$ and
$|\hat f(x_2)-\hat f(x_1)|=2$. Consequently,
\begin{equation}
 2\le |\hat f(x_2)-f(x_2)|+|f(x_1)-\hat f(x_1)|<2,
\end{equation}
a contradiction.

Now suppose $f\notin\gF(\sign;1,1)$. By
Lemma~\ref{lemma:stationary point}, $f\in C^1(\sR)$ and has at most one
stationary point. The assumed error bound and the values of $\hat f$ imply
\begin{equation}
 f(-2)>0,\qquad f(-1)<0,\qquad f(0)>0,\qquad f(1)<0.
\end{equation}
The mean value theorem gives points
$\eta_1\in(-2,-1)$, $\eta_2\in(-1,0)$, and $\eta_3\in(0,1)$ such that
$f'(\eta_1)<0$, $f'(\eta_2)>0$, and $f'(\eta_3)<0$. Since $f'$ is
continuous and $\eta_1<\eta_2<\eta_3$, the intermediate value theorem gives
distinct points $\xi_1\in(\eta_1,\eta_2)$ and
$\xi_2\in(\eta_2,\eta_3)$ satisfying $f'(\xi_1)=f'(\xi_2)=0$. This
contradicts Lemma~\ref{lemma:stationary point}. Hence the uniform error is at
least one.
\end{proof}

\subsection{Proof of Corollary \ref{coro:shallow PLN->phi}}
\label{sec:proof corollary 3.1}

\paragraph{Corollary \ref{coro:shallow PLN->phi}}

Suppose there exists $g\in\gF(\mathrm{LN}[n_s];1,n_s)$ that can represent any $\hat g\in\gF(\phi;1,1)$. Then for any $\hat f\in\gF(\phi;1,N)$ with $\hat f:\sR\to\sR$, there exists a shallow PLN-Net $f \in\gF(\mathrm{PLN}[n_s];1,n_sN)$ such that $f=\hat f$. In short,
\begin{equation}
\gF(\phi;1,N)\subseteq\gF(\mathrm{PLN}[n_s];1,n_sN).
\end{equation}

\begin{proof}
    Let $\hat f \in \gF(\phi;1,N)$ and $f:\sR\to\sR$. Set $\hat f(x) = \vw^\top_2 \phi(\vw_1 x+\vb_1)+b_2$. Note that $\phi$ is an element-wise activation function, we then have
    \begin{equation}
        \hat f(x) = \vw^\top_2 \phi(\vw_1 x+\vb_1)+b_2= \sum_{i=1}^N w_{2i}\phi(w_{1i}x+b_{1i})+b_{2i}, 
    \end{equation}
    where
    \begin{equation}
        \vw_1=\mat{w_{11}\\\vdots\\w_{1N}}, \vw_2=\mat{w_{21}\\\vdots\\w_{2N}}, \vb_1=\mat{b_{11}\\\vdots\\b_{1N}},b_{2}=\sum_{i=1}^{N} b_{2i}. 
    \end{equation}

    That means $\hat f$ can be represented by the sum of N shallow $\phi$-Nets $\hat f_i\in\gF(\phi;1,1)$. By Lemma \ref{lemma:shallow LN->phi}, each $f_i$ can be represented by a shallow LN-Net $f_i\in\gF(\mathrm{LN}[n_s];1,n_s)$. By Lemma \ref{lemma:PLN->LN}, the sum of N shallow LN-Nets $f_i\in\gF(\mathrm{LN}[n_s];1,n_s)$ can be represented by a shallow PLN-Net $f \in\gF(\mathrm{PLN}[n_s];1,n_sN)$. Therefore, $f$ can represent $\hat f$. 

    Here we complete the proof. 
\end{proof}

Similarly, we can extend the corollary to the case $\hat f: \sR^d \to\sR^m$. 

\begin{corollary}
\label{Corollary:A4}
    Suppose every width-one shallow $\phi$-Net can be represented by an
    $\mathrm{LN}[n_s]$-Net. Then for any $\hat f\in\gF(\phi;1,N)$ with
    $\hat f:\sR^d\to\sR^m$, there exists a shallow PLN-Net
    $f \in\gF(\mathrm{PLN}[n_s];1,n_sN)$ such that $f=\hat f$. In short,
    \begin{equation}
    \gF(\phi;1,N)\subseteq\gF(\mathrm{PLN}[n_s];1,n_sN).
    \end{equation}
\end{corollary}

\begin{proof}
    Let $\hat f \in \gF(\phi;1,N)$ with
    $\hat f:\sR^d\to\sR^m$. Write
    $\hat f(\rvx) = \mW_2 \phi(\mW_1 \rvx+\vb_1)+\vb_2$. Since $\phi$
    acts element-wise,
    \begin{equation}
        \hat f(\rvx)
        = \sum_{i=1}^N
        \left[\vw_{2i}\phi(\vw_{1i}^\top\rvx+b_{1i})+\vb_{2i}\right],
    \end{equation}
    where
    \begin{equation}
        \mW_1=\mat{\vw_{11}^\top\\\vdots\\\vw_{1N}^\top},\quad
        \mW_2=\mat{\vw_{21}&\cdots & \vw_{2N}},\quad
        \vb_1=\mat{b_{11}\\\vdots\\b_{1N}},\quad
        \vb_{21}=\vb_2,\quad \vb_{2i}=\bm0\ (i\ge2).
    \end{equation}

    Thus $\hat f$ is a sum of $N$ width-one shallow $\phi$-Nets from
    $\sR^d$ to $\sR^m$. Corollary~\ref{corollary:A3} represents each summand
    by an $\mathrm{LN}[n_s]$-Net, and Lemma~\ref{lemma:PLN->LN} combines
    them into one shallow $\mathrm{PLN}[n_s]$-Net of width $n_sN$.

    Here we complete the proof. 
\end{proof}

\subsection{Proof of Theorem \ref{theorem:Lipschitz-UAT}}
\label{sec:proof thm 3.2}

\paragraph{Theorem \ref{theorem:Lipschitz-UAT}}(Universal approximations of univariate $L$-Lipschitz continuous functions by shallow PLN-Nets).
    Given any $L$-Lipschitz continuous function $\hat f\in C([0,1])$ and
    every $\epsilon>0$, there exists a shallow PLN-Net
    $f\in\gF(\mathrm{PLN}[n_s];1,n_sN)$ with $n_s\ge2$ and
    $N=\floor{L/(2\epsilon)}+1$, such that
    \begin{equation}
        |f-\hat{f}|_{L^\infty}< \epsilon
    \end{equation}
\begin{proof}
    Step 1: construct a proper $\bar{f}\in\gF(\mathrm{sign};1,N)$ to approximate $\hat{f}$, namely $|\bar{f}-\hat{f}|_{L^{\infty}}<\epsilon$. 

    We construct a function 
\begin{equation}
    \bar{f}(x) = \sum_{j=1}^N \alpha_j \sign(w_j x + b_j),
\end{equation}
where $N = \floor{L/(2\epsilon)} + 1$. For $1 \le j \le N-1$, we set
\begin{equation}
    \alpha_j = \frac12\left[\hat{f}\!\left(\frac{2j+1}{2N}\right) - \hat{f}\!\left(\frac{2j-1}{2N}\right)\right];
\end{equation}
and for $j = N$,
\begin{equation}
    \alpha_N = \frac12\left[\hat{f}\!\left(\frac{1}{2N}\right) + \hat{f}\!\left(\frac{2N-1}{2N}\right)\right].
\end{equation}
Moreover, we take $w_j = 1$ for all $1 \le j \le N$, $b_j = -\frac{j}{N}$ for $1 \le j \le N-1$, and $b_N = 1$.

With this construction, we can verify the following:
\begin{enumerate}
    \item For $\frac{j-1}{N} < x < \frac{j}{N}$ with $1 \le j \le N$,
        \begin{equation}
            \bar{f}(x) = \hat{f}\!\left(\frac{2j-1}{2N}\right).
        \end{equation}
    \item For $x = \frac{j}{N}$ with $1 \le j \le N-1$,
        \begin{equation}
            \bar{f}(x) = \frac12\left[\hat{f}\!\left(\frac{2j-1}{2N}\right) + \hat{f}\!\left(\frac{2j+1}{2N}\right)\right].
        \end{equation}
    \item Additionally, $\bar{f}(0) = \hat{f}\!\left(\frac{1}{2N}\right)$ and $\bar{f}(1) = \hat{f}\!\left(\frac{2N-1}{2N}\right)$.
\end{enumerate}

Since $N = \floor{L/(2\epsilon)} + 1$, we have $N > L/(2\epsilon)$. Consequently, we obtain the following error estimates:

\begin{enumerate}
    \item For $x = 0$,
        \begin{equation}
            \begin{aligned}
                |\bar{f}(0) - \hat{f}(0)| 
                &= \left|\hat{f}(0) - \hat{f}\!\left(\frac{1}{2N}\right)\right| 
                &\le \frac{L}{2N} < \epsilon.
            \end{aligned}
        \end{equation}
    
    \item For $x = 1$,
        \begin{equation}
            \begin{aligned}
                |\bar{f}(1) - \hat{f}(1)| 
                &= \left|\hat{f}(1) - \hat{f}\!\left(\frac{2N-1}{2N}\right)\right| 
                &\le \frac{L}{2N} < \epsilon.
            \end{aligned}
        \end{equation}
    
    \item For $x = \frac{j}{N}$ with $1 \le j \le N-1$,
        \begin{equation}
            \begin{aligned}
                \left|\bar{f}\!\left(\frac{j}{N}\right) - \hat{f}\!\left(\frac{j}{N}\right)\right|
                &= \left|\frac12\left[\hat{f}\!\left(\frac{2j-1}{2N}\right) 
                    + \hat{f}\!\left(\frac{2j+1}{2N}\right)\right] 
                    - \hat{f}\!\left(\frac{j}{N}\right)\right| \\
                &\le \frac12\left|\hat{f}\!\left(\frac{2j-1}{2N}\right) 
                    - \hat{f}\!\left(\frac{j}{N}\right)\right|
                    + \frac12\left|\hat{f}\!\left(\frac{2j+1}{2N}\right) 
                    - \hat{f}\!\left(\frac{j}{N}\right)\right| \\
                &\le \frac{L}{4N} + \frac{L}{4N} = \frac{L}{2N} < \epsilon.
            \end{aligned}
        \end{equation}
    
    \item For $\frac{j-1}{N} < x < \frac{j}{N}$ with $1 \le j \le N$,
        \begin{equation}
            \begin{aligned}
                |\bar{f}(x) - \hat{f}(x)| 
                &= \left|\hat{f}\!\left(\frac{2j-1}{2N}\right) - \hat{f}(x)\right| 
                &\le L\left|\frac{2j-1}{2N} - x\right| 
                &< \frac{L}{2N} < \epsilon.
            \end{aligned}
        \end{equation}
\end{enumerate}

In summary, for every $x \in [0,1]$, one of the above four cases applies, and we have
\begin{equation}
    |\bar{f}(x) - \hat{f}(x)|_{L^\infty([0,1])} < \epsilon.
\end{equation}
Thus, the constructed function $\bar{f}(x) = \sum_{j=1}^N \alpha_j \sign(x + b_j)$ uniformly approximates $\hat{f}$ on $[0,1]$ with error less than $\epsilon$.

Step 2: by Corollary \ref{coro:shallow PLN->phi} and Lemma \ref{lemma:shallow LN->sign}, since $\bar{f}\in \gF(\mathrm{sign};1,N)$, there exists a shallow PLN-Net $f \in\gF(\mathrm{PLN}[n_s];1,n_sN)$, such that $f=\bar{f}$. 
    
Step 3: Since $f=\bar{f}$ and $|\bar{f}(x) - \hat{f}(x)|_{L^\infty([0,1])} < \epsilon$, we obtain $|f-\hat f|_{L^\infty}<\epsilon$. 

Here we complete the proof. 
\end{proof}

\subsection{Stable Layer Normalization with a Numerical Constant}
\label{sec:proof thm 3.2-2}

The explicit width bound in Theorem~\ref{theorem:Lipschitz-UAT} uses the
idealized zero-variance convention, because it represents the sign function
exactly. We now consider the standard stabilized operator
\begin{equation}
\label{eqn:stable LN appendix}
 [\mathrm{LN}_{\varepsilon}(\rvh)]_j
 =
 \frac{h_j-\mu}{\sqrt{\sigma^2+\varepsilon}},
 \qquad \varepsilon>0,
\end{equation}
and show that it can still represent the smooth activation
$\phi(z)=z/\sqrt{1+z^2}$ exactly. We use
$\mathrm{PLN}_{\varepsilon}[n_s]$ to denote the direct sum of
$\mathrm{LN}_{\varepsilon}[n_s]$ blocks.

\begin{lemma}[Exact smooth activation under stabilized LN]
\label{lemma:stable LN->phi}
Let $\phi(z)=z/\sqrt{1+z^2}$ and fix $\varepsilon>0$. For every
$n_s\ge3$, any width-one shallow $\phi$-Net can be represented exactly by a
width-$n_s$ shallow $\mathrm{LN}_{\varepsilon}[n_s]$-Net.
\end{lemma}

\begin{proof}
Let a width-one $\phi$-Net be
\begin{equation}
 \hat f(\rvx)=\alpha\phi(\vw^\top\rvx+b)+\beta.
\end{equation}
We construct the corresponding stabilized LN-Net in three steps.

Step 1: Construct the first affine layer.
For $n_s\ge3$, take
\begin{equation}
 \rva=\sqrt{\frac{n_s}{2}}(1,-1,0,\ldots,0)^\top,
 \qquad
 \rvc=\sqrt{\frac{n_s}{6}}(1,1,-2,0,\ldots,0)^\top.
\end{equation}
They satisfy
\begin{equation}
 \bm1^\top\rva=\bm1^\top\rvc=0,
 \qquad
 \rva^\top\rvc=0,
 \qquad
 \|\rva\|_2^2=\|\rvc\|_2^2=n_s.
\end{equation}
Set $A=\sqrt{1+\varepsilon}$ and $z=\vw^\top\rvx+b$. The first affine
layer outputs
\begin{equation}
 \rvh(z)=A\rva z+\rvc.
\end{equation}

Step 2: Apply stabilized LN.
Since both $\rva$ and $\rvc$ are centered, the mean of $\rvh(z)$ is zero.
Moreover, its variance is
\begin{equation}
 \sigma^2
 =\frac{1}{n_s}\|A\rva z+\rvc\|_2^2
 =A^2z^2+1.
\end{equation}
Let $\rvu=\rva/n_s$. By the orthogonality of $\rva$ and $\rvc$, we obtain
\begin{align}
 \rvu^\top\mathrm{LN}_{\varepsilon}(\rvh(z))
 &=
 \frac{Az}{\sqrt{A^2z^2+1+\varepsilon}}
 =
 \frac{Az}{A\sqrt{z^2+1}}
 =
 \phi(z).
\end{align}

Step 3: Construct the output affine layer.
The output layer applies $\alpha\rvu^\top$ and adds $\beta$. Therefore,
\begin{equation}
 \alpha\rvu^\top
 \mathrm{LN}_{\varepsilon}(\rvh(\vw^\top\rvx+b))+\beta
 =\alpha\phi(\vw^\top\rvx+b)+\beta
 =\hat f(\rvx).
\end{equation}
Thus, the stabilized LN-Net represents the width-one $\phi$-Net exactly.
This completes the proof.
\end{proof}

We next extend the construction from one activation unit to shallow and deep
networks. Applying Lemma~\ref{lemma:stable LN->phi} independently to $N$
activation units and arranging the resulting LN blocks in parallel gives
\begin{equation}
 \gF(\phi;1,N)
 \subseteq
 \gF(\mathrm{PLN}_{\varepsilon}[n_s];1,n_sN),
 \qquad n_s\ge3.
\end{equation}
For a deep $\phi$-Net, we apply the same construction to every activation
layer and combine the adjacent affine layers, as in the proof of
Lemma~\ref{lemma:deep PLN-Net to phi}. Hence the corresponding inclusion also
holds for arbitrary depth.

Since $\phi$ is smooth and non-polynomial, standard shallow-network
universal-approximation results
\citep{hornik1991approximation,pinkus1999approximation} imply that stabilized
PLN-Nets are dense on compact domains. The multivariate approximation rate,
the deep-network approximation result, and the position-wise sequence and
convolutional representation results are all transferred through this exact
$\phi$-representation. Therefore, these results continue to hold under the
stabilized convention.

The stabilized operator is continuous and does not represent the discontinuous
sign function exactly. Consequently, this argument does not assert the
explicit width bound
$N=\floor{L/(2\epsilon)}+1$ from
Theorem~\ref{theorem:Lipschitz-UAT} for stabilized LN. Learnable scale and
shift parameters do not alter the conclusion because they can be absorbed into
the following affine layer.

\section{Additional Mathematical Proofs in Section~\ref{sec:extension}}

\subsection{Proof of Lemma \ref{lemma:deep PLN-Net to phi}}
\label{sec:proof 4.1}

\paragraph{Lemma \ref{lemma:deep PLN-Net to phi}} 
Given any $\hat f\in\gF(\phi;L,N)$ with $\hat f:\sR^d\to\sR^m$,
there exists a PLN-Net $f\in \gF(\mathrm{PLN}[n_s];L,n_sN)$ with
$n_s\ge3$ such that $f=\hat f$. In other words, we have
    \begin{equation}
        \gF(\phi;L,N)\subseteq\gF(\mathrm{PLN}[n_s];L,n_sN),
    \end{equation}
    where $\phi(x)=x/\sqrt{x^2+1}$ is an activation function.
\begin{proof}

    Assume $\hat f:\sR^d\to\sR^m$. By Corollary \ref{Corollary:A4}, we obtain that 
    \begin{equation}
    \label{eqn:PLN->phi}
        \gF(\phi;1,N)\subseteq\gF(\mathrm{PLN}[n_s];1,n_sN). 
    \end{equation}

    Now we consider to extend to the deep networks. 

    Given a $\hat f\in\gF(\phi;L,N)$, which can be unrolled as
    \begin{equation}
\begin{aligned}
    \hat f &= W_{L+1} \circ \phi_L \circ W_L \circ \cdots \circ W_2 \circ \phi_1 \circ W_1 \\
    &= W_{L+1} \circ \phi_L \circ I_L  \circ W_L \circ \cdots \circ W_2 \circ \phi_1 \circ I_1 \circ W_1 \\
    &= (W_{L+1} \circ \phi_L \circ I_L)  \circ W_L \circ \cdots \circ (W_2 \circ \phi_1 \circ I_1 \circ W_1). 
\end{aligned}
\end{equation}
where $W_1:\mathbb{R}^{d_x}\to\mathbb{R}^{d_1}$,
$W_\ell:\mathbb{R}^{d_{\ell-1}}\to\mathbb{R}^{d_\ell}$ for
$\ell=2,\ldots,L$, and
$W_{L+1}:\mathbb{R}^{d_L}\to\mathbb{R}^{d_y}$ are affine layers.
$I_\ell:\sR^{d_\ell}\to\sR^{d_\ell}$ for $\ell=1,\ldots,L$ are identity
maps.

Now we construct that $\hat f_\ell=W_{\ell+1} \circ \phi_\ell \circ I_\ell$ for $\ell=2,\ldots,L$ and $\hat f_1=W_2\circ \phi_1 \circ I_1 \circ W_1$. We find that each $\hat f_\ell \in \gF(\phi;1,d_\ell)$ and
\begin{equation}
    \hat f=\hat f_L\circ\hat f_{L-1}\circ\cdots\circ\hat f_1.
\end{equation}

By \Eqn\ref{eqn:PLN->phi}, we can find $f_\ell \in\gF(\mathrm{PLN}[n_s];1,n_sd_\ell)$ such that $f_\ell=\hat f_\ell$. Set $f_\ell = V_{\ell}\circ \mathrm{PLN}\circ V_{\ell}'$, then we construct 
\begin{equation}
    \begin{aligned}
        f &= f_L \circ f_{L-1}\circ\cdots\circ f_1 \\
        &= (V_{L}\circ \mathrm{PLN}\circ V_{L}' )\circ (V_{L-1}\circ \mathrm{PLN}\circ V_{L-1}')\circ\cdots\circ (V_{1}\circ \mathrm{PLN}\circ V_{1}') \\
        &= V_{L}\circ \mathrm{PLN}\circ V_{L}' \circ V_{L-1}\circ \mathrm{PLN}\circ V_{L-1}'\circ\cdots\circ V_{1}\circ \mathrm{PLN}\circ V_{1}' \\
        &= \bar V_{L+1}\circ\mathrm{PLN}\circ \bar V_L\circ \cdots\circ \mathrm{PLN}\circ \bar V_1
    \end{aligned}
\end{equation}
where $\bar V_1 =V_1':\mathbb{R}^{d_x}\to\mathbb{R}^{n_sd_1}$,
$\bar V_\ell=V_\ell'\circ V_{\ell-1}:
\mathbb{R}^{n_sd_{\ell-1}}\to\mathbb{R}^{n_sd_\ell}$ for
$\ell=2,\ldots,L$, and
$\bar V_{L+1}=V_L:\mathbb{R}^{n_sd_L}\to\mathbb{R}^{d_y}$ are affine
layers.

Thus $f$ is a PLN-Net of depth $L$. Let
$D=\max\{d_1,\ldots,d_L\}\le N$ be the actual width of $\hat f$. The width
of $f$ is $\max\{n_sd_1,\ldots,n_sd_L\}=n_sD\le n_sN$, so
$f\in\gF(\mathrm{PLN}[n_s];L,n_sN)$. Since each
$f_\ell=\hat f_\ell$, we also have $f=\hat f$, completing the proof for the
stated vector input and output dimensions.
\end{proof}

\subsection{Proof of Theorem \ref{thm:shallow PLN->universal}}
\label{sec:proof thm 4.1}

We first recall the definition of Sobolev space. 

    \paragraph{Sobolev space $\mathcal{W}^{s,\infty}(\Omega)$.}
    Let $\Omega \subset \mathbb{R}^n$ be a bounded Lipschitz domain, $s \in \mathbb{N}$. 
    For a multi-index $\alpha = (\alpha_1, \dots, \alpha_n) \in \mathbb{N}_0^n$, 
    let $|\alpha| = \sum_{i=1}^n \alpha_i$ and 
    $D^\alpha u = \partial^{|\alpha|} u / \partial x_1^{\alpha_1} \cdots \partial x_n^{\alpha_n}$ 
    denote weak derivatives.

    The Sobolev space is defined as
    \begin{equation}
        \mathcal{W}^{s,\infty}(\Omega) = \big\{ u \in L^\infty(\Omega) \mid 
        D^\alpha u \in L^\infty(\Omega)\ \forall\, |\alpha| \leq s \big\}, 
    \end{equation}
    with norm
    \begin{equation}
        \|u\|_{\mathcal{W}^{s,\infty}(\Omega)} = 
        \max_{|\alpha| \leq s} \|D^\alpha u\|_{L^\infty(\Omega)} .
    \end{equation}

\paragraph{Theorem \ref{thm:shallow PLN->universal}}
Let $d\ge2$, $s\ge1$, and let $\hat f:B^d\to\sR$ belong to $\mathcal{W}^{s,\infty}(B^d)$ with $\|\hat f\|_{\mathcal{W}^{s,\infty}}\le1$, where $B^d=\{\rvx\in\sR^d:\|\rvx\|\le1\}$. For every $N\ge1$, there exists a shallow PLN-Net $f \in \gF(\mathrm{PLN}[n_s];1,n_sN)$ with $n_s\ge3$ such that
	\begin{equation}
		\|f-\hat f\|_{L^{\infty}(B^d)} \le CN^{-s/d}, 
	\end{equation}
    for some $C$ independent of $N$. 

Theorem~\ref{thm:shallow PLN->universal} follows from
\citet[Theorem~6.8]{pinkus1999approximation}. Here we fix $p=\infty$ and
rewrite the result as follows.

\begin{theorem}
[\citealp{pinkus1999approximation}, Theorem~6.8]
    \label{theorem:thm 6.8}
    
    Assume $\sigma :\sR\to\sR$ is such that $\sigma\in C^{\infty}(\Theta)$
    on some open interval $\Theta$, and $\sigma$ is not a polynomial on
    $\Theta$. Then for every $m\ge 1$ and $n\ge 2$,
    \begin{equation}
        \sup_{\hat f\in \gB^m_\infty}\inf_{f\in\gF(\sigma;1,r)}|f-\hat f|_{L^\infty(B^n)}\le Cr^{-m/n}
    \end{equation}
    for some constant $C$ independent of $r$, where 
    \begin{equation}
        \gB^m_\infty=\{f\in W^{m,\infty}(B^n):\|f\|_{W^{m,\infty}}\le1\},B^n=\{\rvx\in\sR^n,\|\rvx\|\le1\}. 
    \end{equation}
\end{theorem}

Here we give the proof of Theorem \ref{thm:shallow PLN->universal}.

\begin{proof}
    Note that $\phi(x)=x/\sqrt{x^2+1}$ is such that $\phi\in C^{\infty}(-1,1)$ on some open interval $\Theta$, and $\phi$ is not a polynomial on $(-1,1)$. 
    
    Given $\hat f\in \mathcal{W}^{s,\infty}(B^d)$ and $\|\hat f\|_{\mathcal{W}^{s,\infty}}\le1$, by Theorem \ref{theorem:thm 6.8}, there exists some $\bar f\in\gF(\phi;1,N)$ such that
    \begin{equation}
        \|\bar f- \hat f\|_{L^{\infty}(B^d)}\le CN^{-s/d}. 
    \end{equation}
    By Corollary~\ref{coro:shallow PLN->phi}, there exists a shallow PLN-Net $f \in \gF(\mathrm{PLN}[n_s];1,n_sN)$ such that $f=\bar f$.

    Therefore, $\|f- \hat f\|_{L^{\infty}(B^d)}\le CN^{-s/d}$, which completes the proof.
\end{proof}

\subsection{Proof of Theorem \ref{thm:deep PLN->ReLU}}
\label{sec:proof thm 4.2}

\paragraph{Theorem \ref{thm:deep PLN->ReLU}}
(Approximate ReLU networks by PLN-Nets.)
Let $\hat f \in \gF(\mathrm{ReLU};L,N)$ with
$\hat f:[0,1]^d\to\sR^n$, and let $\epsilon>0$.
Then there exists a PLN-Net
$f \in \gF(\mathrm{PLN}[n_s];2L,3n_sN)$ with $n_s\ge3$ such that
\begin{equation}
    \|f-\hat f\|_{L^\infty([0,1]^d)} < \epsilon .
\end{equation}

Theorem~\ref{thm:deep PLN->ReLU} follows from
\citet[Theorem~1]{zhang2024deep}.

\begin{theorem}
[\citealp{zhang2024deep}, Theorem~1]
\label{theorem:original theorem 1 of ReLU}
    Suppose ${\rho}\in \mathscr{A}$ and
    $\hat f_{\mathrm{ReLU}}\in\gF(\mathrm{ReLU};L,N):\sR^d\to\sR^n$.
    Then for any $\epsilon>0$ and $A>0$, there exists
    $f_\rho\in\gF(\rho;2L,3N):\sR^d\to\sR^n$ such that
    \begin{equation}
        \|f_\rho-\hat f_{\mathrm{ReLU}}\|_{\sup([-A,A]^d)}<\epsilon. 
    \end{equation}
    In the notation of \citet{zhang2024deep},
    $\mathscr{A}_3\subset\mathscr{A}$, where
    \begin{equation}
        \mathscr{A}_3=\{\rho:\sR\to\sR|\rho\in\mathscr{S},\exists x_0\in\sR,\rho''(x_0)\ne 0\}, 
    \end{equation}
    where
    \begin{equation}
        \mathscr{S}=\{h:\sR\to\sR\mid \sup_{x\in\sR}|h(x)|<\infty,
        \ h(+\infty),h(-\infty)\in\sR,\ h(+\infty)\ne h(-\infty) \}.
    \end{equation}
\end{theorem}

Now we prove Theorem \ref{thm:deep PLN->ReLU}.

\begin{proof}
    Let $\hat f\in\gF(\mathrm{ReLU};L,N)$. For
    $\phi(x)=x/\sqrt{x^2+1}$, we have
    $\phi(-\infty)=-1\ne\phi(+\infty)=1$ and
    $\sup_{x\in\sR}|\phi(x)|=1$, so $\phi\in\mathscr{S}$. Moreover,
    $\phi''(1)\ne0$, hence $\phi\in\mathscr{A}_3\subset\mathscr{A}$.

    Therefore, by Theorem \ref{theorem:original theorem 1 of ReLU}, there
    exists $\bar f\in\gF(\phi;2L,3N)$ such that
    $\|\bar f-\hat f\|_{\sup([-1,1]^d)}<\epsilon$. By
    Lemma~\ref{lemma:deep PLN-Net to phi}, there exists a PLN-Net
    $f \in \gF(\mathrm{PLN}[n_s];2L,3n_sN)$ with $n_s\ge3$ such that
    $f=\bar f$. Furthermore, we obtain
    \begin{equation}
        \|f-\hat f\|_{L^\infty([0,1]^d)} \le \|f-\hat f\|_{L^\infty([-1,1]^d)} \le \|f-\hat f\|_{\sup([-1,1]^d)}<\epsilon.
    \end{equation}
    Here we complete the proof. 
\end{proof}

\section{Mathematical Proofs in Section~\ref{sec:4}}

\subsection{Proof of Theorem \ref{thm:transformer}}
\label{sec:proof thm 4.6}

\paragraph{Theorem \ref{thm:transformer}}
    (Exact representation of position-wise FFNs by PLN).
    %
    Let $\hat\varphi_1$ and $\hat\varphi_2$ be two sequence linear layers mapping token dimensions $d_x\to N$ and $N\to d_y$, respectively, and let $\phi(x)=x/\sqrt{x^2+1}$. Define the position-wise FFN mapping $\hat f = \hat\varphi_2 \circ \phi \circ \hat\varphi_1 $. For every $n_s\ge3$, there exist sequence linear layers $\varphi_1$ and $\varphi_2$ mapping $d_x\to n_sN$ and $n_sN\to d_y$, respectively, such that
    \begin{equation}
        f = \varphi_2\circ\mathrm{PLN}[n_s]\circ\varphi_1
    \end{equation}
    exactly represents $\hat f$ for all sequence inputs.

\begin{proof}
    Assume that $\hat\varphi_1(\mX)=\mX\mW+\bm1_m\vb^\top$. 

    Note that for $\mX$, we separate it into $\mX^\top=\mat{\rvx_1&\cdots&\rvx_s}^\top$ and obtain that
    \begin{equation}
    \begin{aligned}
        \hat\varphi_1(\mX)&=\mX\mW+\bm1_m\vb^\top \\
        &= \mat{\rvx_1^\top\\\vdots\\\rvx_s^\top}\mW+\mat{\vb^\top\\\vdots\\\vb^\top} \\
        &= \mat{\rvx_1^\top\mW+\vb^\top\\\vdots\\\rvx_s^\top\mW+\vb^\top} \\
        &= \mat{\hat\varphi_1(\rvx_1^\top)\\\vdots\\\hat\varphi_1(\rvx_s^\top)}.
    \end{aligned}
    \end{equation}

    Such calculation is also suitable for $\hat\varphi_2$ Since $\phi$ activates the neurons element-wisely, we can obtain that\
    \begin{equation}
    \label{eqn:transformer unrolling}
        [\hat\varphi_2\circ\phi\circ\hat\varphi_1](\mX)= \mat{[\hat\varphi_2\circ\phi\circ\hat\varphi_1](\rvx_1^\top)
        \\
        \vdots
        \\ [\hat\varphi_2\circ\phi\circ\hat\varphi_1](\rvx_s^\top)}
    \end{equation}
    By Corollary \ref{coro:shallow PLN->phi}, we can find some $\varphi_1:\rvx^\top\to\rvx^\top\mW_1+\vb_1^\top$ and $\varphi_2:\rvx^\top\to\rvx^\top\mW_2+\vb_2^\top$, such that
    \begin{equation}
        [\varphi_2\circ\mathrm{PLN}\circ\varphi_1](\rvx^\top) = [\hat\varphi_2\circ\phi\circ\hat\varphi_1](\rvx^\top), \forall \rvx\in\sR^d.
    \end{equation}
Applying Corollary \ref{coro:shallow PLN->phi} is suitable, for we can identify it by transposing the matrices in the calculation process. 

Similar to \Eqn\ref{eqn:transformer unrolling}, we can identify that
\begin{equation}
\mat{\varphi_1(\rvx_1^\top)\\\vdots\\\varphi_1(\rvx_s^\top)} = \varphi_1(\mX)
\end{equation}
holds for any $\mX\in\sR^{s\times d}$, and the similar conclusion holds for $\varphi_2$. 

Set $\mH=\varphi_1(\mX)$, we figure out that
\begin{equation}
    \mathrm{PLN}(\mH)= \mat{\mathrm{PLN}(\rvh_1^\top)\\ \vdots\\ \mathrm{PLN}(\rvh_s^\top)} = \mat{\mathrm{PLN}(\varphi_1(\rvx_1^\top))\\ \vdots\\ \mathrm{PLN}(\varphi_1(\rvx_s^\top))}.
\end{equation}
Like activation functions, PLN can also act on different dimensionalities---PLN out of the matrix acts on $\sR^{s\times n_sN}$, while the one in the matrix acts on $\sR^{n_sN}$. Both of them have the same norm size $n_s$. 

Furthermore, we figure out that 
\begin{equation}
    [\hat\varphi_2\circ\phi\circ\hat\varphi_1](\mX)= \mat{[\hat\varphi_2\circ\phi\circ\hat\varphi_1](\rvx_1^\top)
        \\
        \vdots
        \\ [\hat\varphi_2\circ\phi\circ\hat\varphi_1](\rvx_s^\top)} = \mat{[\varphi_2\circ\mathrm{PLN}\circ\varphi_1](\rvx_1^\top)
        \\
        \vdots
        \\ [\varphi_2\circ\mathrm{PLN}\circ\varphi_1](\rvx_s^\top)} = [\varphi_2\circ\mathrm{PLN}\circ\varphi_1](\mX). 
\end{equation}
Here we complete the proof. 
\end{proof}

\subsection{Proof of Theorem \ref{thm:cnn}}
\label{sec:proof thm 4.7}

\noindent\textbf{Auxiliary pointwise convolutions.}
Let $\tX\in\sR^{C_{\mathrm{in}}\times H\times W}$ denote an image feature
map. A pointwise convolution $\varphi^{\mathrm{pos}}$ mapping the
channel dimension $C_{\mathrm{in}}\to C_{\mathrm{out}}$ is defined as
\begin{equation}
\label{eqn:position-wise linear}
 [\varphi^{\mathrm{pos}}(\tX)]_{:,i,j}
 =
 \mW\tX_{:,i,j}+\vb,
\end{equation}
for every spatial position $(i,j)$, where
$\mW\in\sR^{C_{\mathrm{out}}\times C_{\mathrm{in}}}$ and
$\vb\in\sR^{C_{\mathrm{out}}}$ are shared across all spatial positions.
This is the standard $1\times1$ convolution: it mixes channels independently
at each position without changing the spatial receptive field.

\begin{lemma}[Folding pointwise convolutions into general convolutions]
\label{lemma:folding position-wise linear}
Let $\mathcal C$ be the affine convolutional layer in
\Eqn\ref{eqn:general convolution}. Let
$\varphi^{\mathrm{pos}}_1(\rvz)=\mP\rvz+\rvc$ and
$\varphi^{\mathrm{pos}}_2(\rvz)=\mQ\rvz$ be pointwise convolutions,
where the latter is bias-free. Then there exists an affine convolutional
layer $\mathcal C'$ with the same kernel support, stride, dilation, and
padding as $\mathcal C$ such that
\begin{equation}
\label{eqn:position-wise folding}
 \mathcal C'
 =
 \varphi^{\mathrm{pos}}_1
 \circ\mathcal C
 \circ\varphi^{\mathrm{pos}}_2.
\end{equation}
In particular, its kernel slices and bias are
\begin{equation}
\label{eqn:position-wise folded parameters}
 K'_\delta=\mP K_\delta\mQ,
 \qquad
 \rvb'=\mP\rvb+\rvc.
\end{equation}
\end{lemma}

\noindent\textbf{Proof of Lemma~\ref{lemma:folding position-wise linear}.}
\begin{proof}
Because the two pointwise convolutions share their parameters across spatial
positions, for every input tensor $\tX$ and output position $(i,j)$,
\begin{align}
 &[\varphi^{\mathrm{pos}}_1\circ\mathcal C
 \circ\varphi^{\mathrm{pos}}_2(\tX)]_{:,i,j}\nonumber\\
 &\quad=
 \mP\left[
 \rvb+\sum_{(u,v)\in S}K_{u,v}\mQ
 \widetilde{\tX}_{:,\,s_hi+d_hu-p_h,\,s_wj+d_wv-p_w}
 \right]+\rvc\nonumber\\
 &\quad=
 (\mP\rvb+\rvc)
 +\sum_{(u,v)\in S}
 (\mP K_{u,v}\mQ)
 \widetilde{\tX}_{:,\,s_hi+d_hu-p_h,\,s_wj+d_wv-p_w}.
\end{align}
This is an affine convolution with kernel slices
$K'_{u,v}=\mP K_{u,v}\mQ$ and bias
$\rvb'=\mP\rvb+\rvc$. Zero padding is preserved because the pre-composed
pointwise convolution is bias-free; the stride and dilation indices are
unchanged.
\end{proof}

\noindent\textbf{A scalar identity.}
Let $n=n_s\ge3$. Choose $\rva,\rvc\in\sR^n$ as
\begin{equation}
\label{eqn:cnn ac}
 \rva=\sqrt{\frac n2}(1,-1,0,\ldots,0)^\top,
 \qquad
 \rvc=\sqrt{\frac n6}(1,1,-2,0,\ldots,0)^\top,
\end{equation}
and set $\rvu=\rva/n$. These vectors satisfy
\begin{equation}
 \bm1^\top\rva=\bm1^\top\rvc=0,
 \qquad
 \rva^\top\rvc=0,
 \qquad
 \|\rva\|_2^2=\|\rvc\|_2^2=n.
\end{equation}
For every $t\in\sR$, the vector $\rva t+\rvc$ has mean zero and variance
$t^2+1$. Consequently,
\begin{equation}
\label{eqn:cnn scalar identity}
 \rvu^\top\mathrm{LN}(\rva t+\rvc)
 =
 \frac{t}{\sqrt{1+t^2}}
 =
 \phi(t).
\end{equation}

\noindent\textbf{A pointwise-convolutional representation of $\phi$.}
For any channel dimension $C$, define two pointwise ($1\times1$) convolutions
$\varphi^{\mathrm{pos}}_{1,C}$ and
$\varphi^{\mathrm{pos}}_{2,C}$ by
\begin{align}
 [\varphi^{\mathrm{pos}}_{1,C}(\tZ)]_{r,g}(p)
 &=a_g\tZ_r(p)+c_g,
 \label{eqn:cnn position-wise encoder}\\
 [\varphi^{\mathrm{pos}}_{2,C}(\tG)]_r(p)
 &=\sum_{g=1}^{n}u_g\tG_{r,g}(p),
 \label{eqn:cnn position-wise decoder}
\end{align}
where $1\le r\le C$ and $1\le g\le n$. The first layer maps
$C\to nC$ channels and the second maps $nC\to C$ channels. By
\Eqn\ref{eqn:cnn scalar identity}, for every tensor $\tZ$, channel $r$,
and spatial position $p$,
\begin{align}
 &[\varphi^{\mathrm{pos}}_{2,C}
 \circ\mathrm{CPLN}_{n}
 \circ\varphi^{\mathrm{pos}}_{1,C}(\tZ)]_r(p)\nonumber\\
 &\quad=
 \sum_{g=1}^{n}u_g
 \big[\mathrm{LN}(\rva\tZ_r(p)+\rvc)\big]_g
 =
 \phi(\tZ_r(p)).
\end{align}
Hence, as position-wise tensor maps,
\begin{equation}
\label{eqn:cnn position-wise representation}
 \phi
 =
 \varphi^{\mathrm{pos}}_{2,C}
 \circ\mathrm{CPLN}_{n}
 \circ\varphi^{\mathrm{pos}}_{1,C}.
\end{equation}

\paragraph{Theorem \ref{thm:cnn}}
(Exact representation of convolutional blocks by CPLN).
Let $\hat{\mathcal C}_1$ and $\hat{\mathcal C}_2$ be arbitrary affine
convolutional layers mapping channel dimensions $C_x\to N$ and
$N\to C_y$, respectively, and let
$\phi(t)=t/\sqrt{1+t^2}$ act element-wise. For every norm size $n_s\ge3$,
there exist affine convolutional layers $\mathcal C_1$ and $\mathcal C_2$
with the same kernel supports, strides, dilations, and padding as
$\hat{\mathcal C}_1$ and $\hat{\mathcal C}_2$, and intermediate width $n_sN$,
such that
\begin{equation}
 \mathcal C_2\circ\mathrm{CPLN}_{n_s}\circ\mathcal C_1
 =
 \hat{\mathcal C}_2\circ\phi\circ\hat{\mathcal C}_1
\end{equation}
for every input tensor.

\begin{proof}
Apply \Eqn\ref{eqn:cnn position-wise representation} with $C=N$:
\begin{align}
 \hat{\mathcal C}_2\circ\phi\circ\hat{\mathcal C}_1
 &=
 \hat{\mathcal C}_2
 \circ\varphi^{\mathrm{pos}}_{2,N}
 \circ\mathrm{CPLN}_{n}
 \circ\varphi^{\mathrm{pos}}_{1,N}
 \circ\hat{\mathcal C}_1.
\end{align}
Define
\begin{equation}
 \mathcal C_1
 =
 \varphi^{\mathrm{pos}}_{1,N}\circ\hat{\mathcal C}_1,
 \qquad
 \mathcal C_2
 =
 \hat{\mathcal C}_2\circ\varphi^{\mathrm{pos}}_{2,N}.
\end{equation}
The layer $\varphi^{\mathrm{pos}}_{2,N}$ is bias-free. Therefore,
Lemma~\ref{lemma:folding position-wise linear} shows that
$\mathcal C_1$ and $\mathcal C_2$ are affine convolutions with the same
spatial specifications as $\hat{\mathcal C}_1$ and
$\hat{\mathcal C}_2$, respectively. Thus
\begin{equation}
 \mathcal C_2\circ\mathrm{CPLN}_{n}\circ\mathcal C_1
 =
 \hat{\mathcal C}_2\circ\phi\circ\hat{\mathcal C}_1
\end{equation}
for every input tensor.
\end{proof}

\subsection{Extension to Deep Convolutional Networks}
\label{sec:deep cnn representation}

\noindent\textbf{Deep convolutional representation.}
For every $n_s\ge3$, replacing each element-wise $\phi$ by its CPLN
realization and folding the intervening pointwise $1\times1$ convolutions into
the adjacent convolutional layers converts any depth-$L$ $\phi$-activated
convolutional network with hidden channel dimensions $C_1,\ldots,C_L$ into an
exactly equivalent CPLN network with dimensions
$n_sC_1,\ldots,n_sC_L$, without changing its depth or spatial specifications.

\begin{proof}
Write the original network as
\begin{equation}
 \hat F=\hat{\mathcal C}_{L+1}\circ\phi\circ\hat{\mathcal C}_{L}
 \circ\cdots\circ\phi\circ\hat{\mathcal C}_1.
\end{equation}
For each hidden channel dimension $C_\ell$, let
$\varphi^{\mathrm{pos}}_{1,C_\ell}$ and
$\varphi^{\mathrm{pos}}_{2,C_\ell}$ be the two layers in
\cref{eqn:cnn position-wise encoder,eqn:cnn position-wise decoder}. Define
\begin{align}
 \mathcal C_1
 &=
 \varphi^{\mathrm{pos}}_{1,C_1}\circ\hat{\mathcal C}_1,
 \label{eqn:cnn folded first}\\
 \mathcal C_\ell
 &=
 \varphi^{\mathrm{pos}}_{1,C_\ell}
 \circ\hat{\mathcal C}_\ell
 \circ\varphi^{\mathrm{pos}}_{2,C_{\ell-1}},
 \qquad 2\le\ell\le L,
 \label{eqn:cnn folded middle}\\
 \mathcal C_{L+1}
 &=
 \hat{\mathcal C}_{L+1}
 \circ\varphi^{\mathrm{pos}}_{2,C_L}.
 \label{eqn:cnn folded output}
\end{align}
By Lemma~\ref{lemma:folding position-wise linear}, every
$\mathcal C_\ell$ is an affine convolution with the same spatial
specifications as $\hat{\mathcal C}_\ell$. Using associativity of
composition and \Eqn\ref{eqn:cnn position-wise representation},
\begin{align}
 F
 &=
 \mathcal C_{L+1}\circ\mathrm{CPLN}_{n}\circ\mathcal C_L
 \circ\cdots\circ\mathrm{CPLN}_{n}\circ\mathcal C_1\nonumber\\
 &=
 \hat{\mathcal C}_{L+1}
 \circ\bigl(
 \varphi^{\mathrm{pos}}_{2,C_L}\circ\mathrm{CPLN}_{n}
 \circ\varphi^{\mathrm{pos}}_{1,C_L}
 \bigr)
 \circ\hat{\mathcal C}_L\circ\cdots\nonumber\\
 &\qquad\circ
 \bigl(
 \varphi^{\mathrm{pos}}_{2,C_1}\circ\mathrm{CPLN}_{n}
 \circ\varphi^{\mathrm{pos}}_{1,C_1}
 \bigr)
 \circ\hat{\mathcal C}_1\nonumber\\
 &=
 \hat{\mathcal C}_{L+1}\circ\phi\circ\hat{\mathcal C}_L
 \circ\cdots\circ\phi\circ\hat{\mathcal C}_1
 =
 \hat F.
\end{align}
The equality holds for every input tensor, and each hidden channel
$C_\ell$ is replaced by $n_sC_\ell$ channels.
\end{proof}
\section{Additional Experimental Details}
\makeatletter
\setlength{\@fptop}{0pt}
\makeatother

Unless stated otherwise, all experiments are implemented in PyTorch
\citep{paszke2019}. The normalization-domain operators and experiment grids are
our implementations; below we cite the external datasets, architectures,
training recipes, and evaluation tools on which individual protocols rely.
The sections follow the order of Section~\ref{sec:experiments}: MLP function
approximation and PINNs, followed by CNNs and ViTs. The VGG-16 initialization
study is included with the CNN experiments as an auxiliary comparison.

\subsection{MLP Experiments}
\subsubsection{Dense Fourier Function Approximation}
\label{sec:dense fourier details}

For each target seed $s$ and dimension $d\in\{4,8\}$, we generate
\begin{equation}
f_{s,d}(\rvx)=8^{-1/2}\sum_{j=1}^{8}a_{s,j}
\sin(\pi\bm{\omega}_{s,j}^\top\rvx+\phi_{s,j}),
\qquad \rvx\in[-1,1]^d,
\end{equation}
where $a_{s,j}\in\{-1,1\}$, $\bm{\omega}_{s,j}\in\{-1,1\}^d$, and
$\phi_{s,j}$ are independently sampled and then fixed. We use five target
seeds and standardize each target using its training values. We independently
sample $10{,}000$, $2{,}000$, and $10{,}000$ inputs for training, validation,
and testing. Every Fourier mode involves all $d$ coordinates, making the
$d=8$ case a denser-interaction stress test rather than merely adding
irrelevant input coordinates.

All MLPs have width $1024$ and either one or four hidden layers. We compare
Identity, LN, PLN[4], PLN[8], ReLU, SiLU, Tanh, and
$\mathrm{SAT}(x)=x/\sqrt{1+x^2}$. LN and PLN use learnable affine parameters,
and no model contains an additional normalization layer. Each run uses Adam
\citep{Kingma2014Adam} for $10{,}000$ updates with batch size $256$ and zero
weight decay. The learning rate is linearly warmed up for $500$ updates and
then cosine-decayed to zero. We evaluate
$\{10^{-4},3\times10^{-4},10^{-3}\}$ and select one shared learning rate for
each dimension--method--depth configuration by its mean validation MSE over
all five targets and three model seeds. We first average the three model seeds
within each target, then report the mean and sample standard deviation across
the five targets in Table~\ref{tab:function_fitting}.

\begin{table}[!htbp]
    \caption{Dense Fourier test MSE (five-target mean $\pm$ standard deviation
    after averaging three model seeds; lower is better).}
    \label{tab:function_fitting}
    \centering
    \footnotesize
    \setlength{\tabcolsep}{4.1pt}
    \begin{tabular}{lcccc}
    \toprule
    & \multicolumn{2}{c}{$d=4$} & \multicolumn{2}{c}{$d=8$} \\
    \cmidrule(lr){2-3}\cmidrule(lr){4-5}
    Nonlinearity & Depth 1 & Depth 4 & Depth 1 & Depth 4 \\ \midrule
    Identity & $0.9982\pm0.0071$ & $0.9979\pm0.0072$ & $0.9989\pm0.0159$ & $0.9972\pm0.0144$ \\
    LN       & $0.9979\pm0.0071$ & $0.8793\pm0.0885$ & $0.9990\pm0.0158$ & $0.9808\pm0.0099$ \\
    PLN[4]   & $\mathbf{0.0048\pm0.0013}$ & $0.0018\pm0.0006$ & $0.1171\pm0.0197$ & $0.1034\pm0.0244$ \\
    PLN[8]   & $0.0079\pm0.0030$ & $0.00075\pm0.00027$ & $\mathbf{0.0803\pm0.0068}$ & $0.0660\pm0.0215$ \\
    ReLU     & $0.0118\pm0.0043$ & $0.00081\pm0.00025$ & $0.1196\pm0.0082$ & $0.0282\pm0.0036$ \\
    SiLU     & $0.9985\pm0.0072$ & $\mathbf{0.00037\pm0.00026}$ & $0.9986\pm0.0156$ & $\mathbf{0.0269\pm0.0143}$ \\
    Tanh     & $0.3574\pm0.0425$ & $0.0014\pm0.0006$ & $0.9384\pm0.0242$ & $0.0345\pm0.0075$ \\
    SAT      & $0.2302\pm0.0127$ & $0.0011\pm0.0003$ & $0.7131\pm0.0359$ & $0.0335\pm0.0079$ \\ \bottomrule
    \end{tabular}
\end{table}

For one hidden layer, PLN is not only trainable but is the best-performing
nonlinearity in both dimensions: PLN[4] is best for $d=4$, while PLN[8] is
best for $d=8$. In contrast, increasing depth substantially improves the
element-wise activations, with SiLU giving the lowest four-layer error. These
results isolate the intended empirical point: PLN supplies useful multivariate
nonlinearity within a shallow block; they do not claim that PLN must outperform
deep compositions of scalar activations.

\FloatBarrier

\subsubsection{Helmholtz PINN}
\label{sec:experiment 2}
Following the physics-informed neural-network framework of
\citet{raissi2019pinns}, let $u_\theta(x)$ approximate the solution to the
one-dimensional Helmholtz equation. Using automatic differentiation, we
compute $u_\theta''(x)$ and form the residual
\begin{equation}
    r_\theta(x)= -u_\theta''(x)+\pi^2 u_\theta(x)- 2\pi^2\sin(\pi x).
\end{equation}
The exact solution is $u(x)=\sin(\pi x)$. On $\Omega=[-1,1]$, we sample
$5{,}000$ interior points and $500$ boundary points and jointly minimize the
mean squared PDE residual and Dirichlet boundary error. The MLPs have depth $6$,
width $128$ or $256$, and no additional normalization. We compare ReLU, Tanh,
SAT, and PLN with group sizes $4,8,16,32$ over five model seeds. We evaluate
loss weights
$(\lambda_{\mathrm{PDE}},\lambda_{\mathrm{BC}})\in
\{(1,10),(1,1),(3,1)\}$. The archived runs used Adam
\citep{Kingma2014Adam} at $3\times10^{-4}$ and SGD
\citep{amari1993backpropagation} at $10^{-4}$ for $6{,}000$ iterations. Although
the launch scripts specified a step scheduler, the inspected training path did
not advance it. Relative $\mathcal L_2$ error against
the exact solution measures global fitting quality. Table~\ref{tab:best_mse_lossweight}
reports the complete loss-weight breakdown; the best result in each row is
bold and the second best is underlined. Because the objective contains a
second-order differential operator, this is a practical PINN comparison, not
evidence for a derivative-approximation theorem.

    \begin{table}[!htbp]
    \caption{Best observed relative $\mathcal{L}_2$ error for the Helmholtz
    PINN across recorded seeds and optimizer settings. Lower is better;
    the best and second-best results at each width and loss weight are bold
    and underlined, respectively.}
    \label{tab:best_mse_lossweight}
    \centering
    \small
    \setlength{\tabcolsep}{5pt}
    \begin{tabular}{rrccccccc}
    \toprule
    Width                & Loss weight & ReLU     & PLN[4]              & PLN[8]     & PLN[16]    & PLN[32]    & Tanh              & SAT               \\ \midrule
    \multirow{3}{*}{128} & $(1,10)$    & 0.999601 & \textbf{0.000035} & 0.001120 & 0.001176 & 0.007878 & 0.000250          & \underline{0.000104}          \\
     & $(1,1)$     & 0.995612 & \textbf{0.000941} & 0.002284 & \underline{0.000970} & 0.004231 & 0.003733          & 0.000979          \\
     & $(3,1)$     & 0.993593 & 0.009553          & 0.003912 & 0.004856 & 0.007465 & \textbf{0.000440} & \underline{0.000525}          \\
    \multirow{3}{*}{256} & $(1,10)$ & 0.998639 & \textbf{0.000036} & 0.004040 & 0.001631 & 0.000567 & 0.000188 & \underline{0.000052} \\
    & $(1,1)$ & 0.998485 & \underline{0.001416} & 0.005552 & 0.001839 & 0.002643 & 0.002653 & \textbf{0.001219} \\
    & $(3,1)$ & 0.999727 & 0.018451 & 0.002966 & \underline{0.002875} & 0.005476 & 0.009267 & \textbf{0.000410} \\
    \bottomrule
    \end{tabular}

    \end{table}

\subsubsection{Exploratory output-NTK diagnostics}
\label{sec:pinn_ntk_diagnostics}
To examine whether grouping changes the spectrum of the empirical neural
tangent kernel (NTK) \citep{jacot2018neural}, we ran a matched finite-model
comparison in a Helmholtz PINN. This experiment is
\emph{not} a rerun of
Table~\ref{tab:best_mse_lossweight}: it uses $5{,}000$ interior and $100$
boundary training points, loss weights $(1,10)$, and widths $128$ and $256$.
For each method we used Adam learning rates
$\{10^{-5},3\times10^{-5},10^{-4},3\times10^{-4}\}$ and seeds $1,2,4$.
The two training schedules were $4{,}000$ Adam steps and $2{,}800$ Adam steps
followed by at most $1{,}200$ L-BFGS steps. The kernel was measured on $500$
interior points at initialization. Writing
$J_{i,:}=\nabla_{\theta}u_{\theta}(x_i)$, we formed the \emph{output} kernel
$K=JJ^{\mathsf T}$; its 90\% effective rank is the fewest leading
eigenvalues whose sum reaches 90\% of $\operatorname{tr}(K)$.

\begin{table}[!htbp]
    \caption{Exploratory Helmholtz PINN diagnostics. Each entry is the median
    over 24 matched width--learning-rate--seed configurations. The effective
    rank is measured at initialization; lower relative $\mathcal L_2$ error
    is better. This protocol differs from Table~\ref{tab:best_mse_lossweight}.}
    \label{tab:pinn_output_ntk}
    \centering
    \small
    \begin{tabular}{lrrr}
    \toprule
    Method & Initial 90\% rank & Adam error & Adam + L-BFGS error \\
    \midrule
    ReLU   & 1  & 0.99999    & 1.00000 \\
    Tanh   & 1  & 0.0018235  & $2.2481\times10^{-6}$ \\
    SAT    & 1  & 0.00067438 & $1.2747\times10^{-6}$ \\
    PLN[4] & 12 & 0.00067266 & $4.8682\times10^{-6}$ \\
    LN     & 2  & 0.64416    & 0.62834 \\
    \bottomrule
    \end{tabular}
\end{table}

PLN[4] has a higher median initial output-kernel effective rank than the
listed baselines and a low median error with Adam alone. However, SAT has
the lowest median error after the L-BFGS stage despite its lower initial
rank. The 90\% effective rank summarizes spectral mass; it does not measure
the smallest nonzero eigenvalue or establish better kernel conditioning.
These results therefore show neither a monotone rank--error relation nor a
general optimization advantage. Because $J$ differentiates the network
output rather than the PDE residual or boundary loss, the measurements
cannot identify an NTK mechanism for the approximation theorems or for the
distinct Table~\ref{tab:best_mse_lossweight} experiment.

\FloatBarrier
\subsection{CNN Experiments}
\subsubsection{Mixed-Axis Normalization}
\label{sec:cnn grouping details}

We use the full CIFAR-10 training split \citep{krizhevsky2009learning} and
evaluate both fixed random labels, following the established random-label
fitting protocol \citep{2017_ICLR_Zhang}, and the original labels. The network processes
feature maps in
$\mathbb R^{B\times128\times32\times32}$ using
$[\mathrm{Conv}(3\!\times\!3)\text{--}\mathrm{Norm}]^4$, followed by
$4\times4$ average pooling and a linear classifier. There is no separate
element-wise activation. If $G_C$ is the number of channel groups and each
spatial axis is split into $K$ parts, each normalization domain has shape
$(128/G_C)\times(32/K)\times(32/K)$.

We evaluate all $8\times6=48$ choices
$G_C\in\{1,2,4,8,16,32,64,128\}$ and
$K\in\{1,2,4,8,16,32\}$. Every configuration is run with three model seeds for
$200$ epochs using SGD \citep{amari1993backpropagation}, learning rate $0.01$,
momentum $0.9$, and batch size
$128$. Weight decay is zero and normalization affine parameters are disabled.
For the clean-label experiment, final training and test accuracy are reported;
for the random-label experiment, final fitting accuracy is measured on the
same fixed $50{,}000$-example training set.

\begin{figure}[ht]
    \centering
    \includegraphics[width=\linewidth]{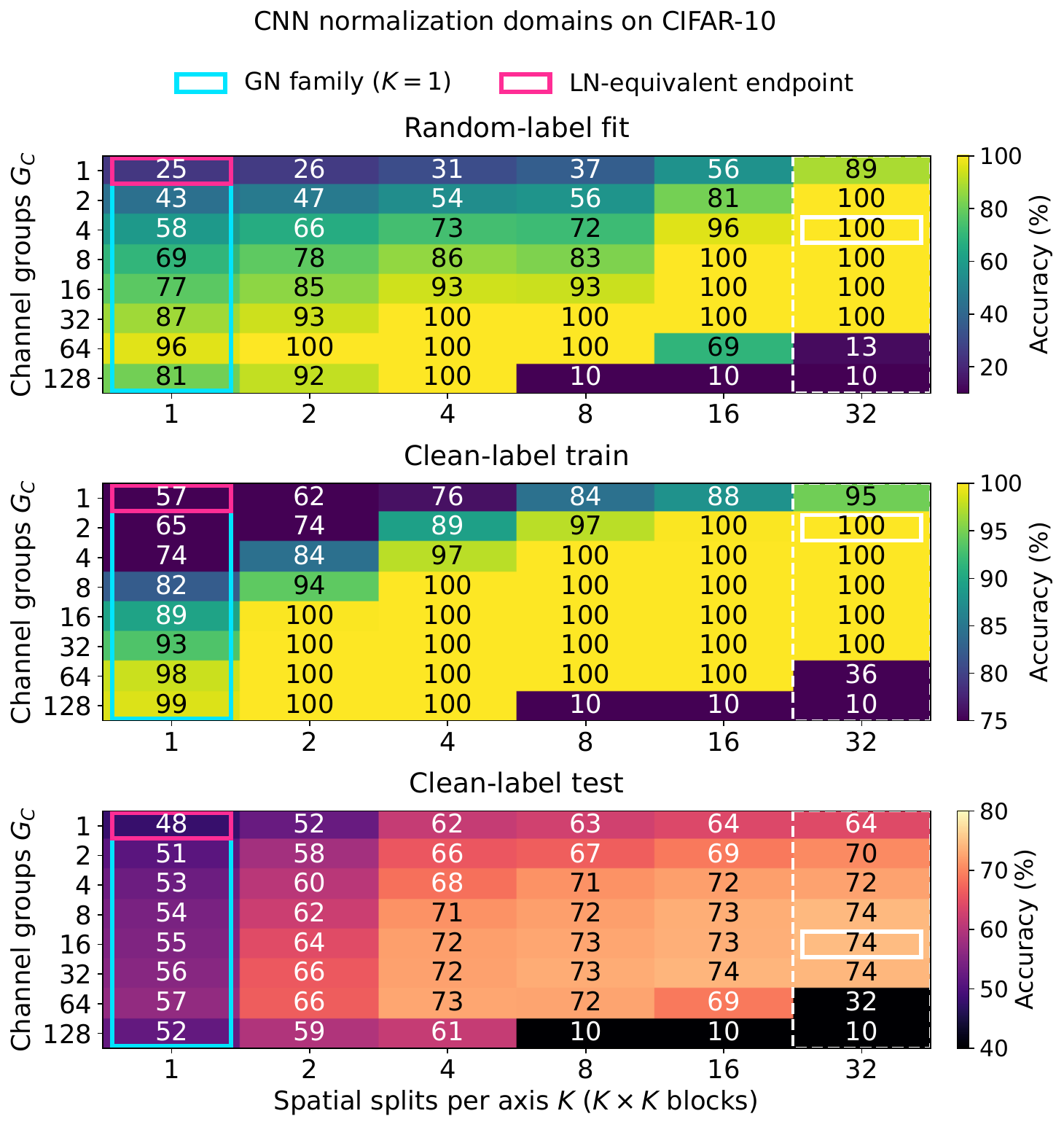}
    \caption{CIFAR-10 accuracy across mixed-axis domains in the depth-$4$,
    width-$128$ CNN (three-seed means): random-label fit (top), clean-label
    train (middle), and clean-label test (bottom). Cyan marks GN ($K=1$),
    magenta marks LN, dashed $K=32$ marks position-wise PLN, and white boxes
    mark panel maxima. Each cell shows its three-seed mean accuracy, rounded
    to the nearest percentage point.}
    \label{fig:cnn_grouping}
\end{figure}

\begin{figure}[H]
    \centering
    \includegraphics[width=\linewidth]{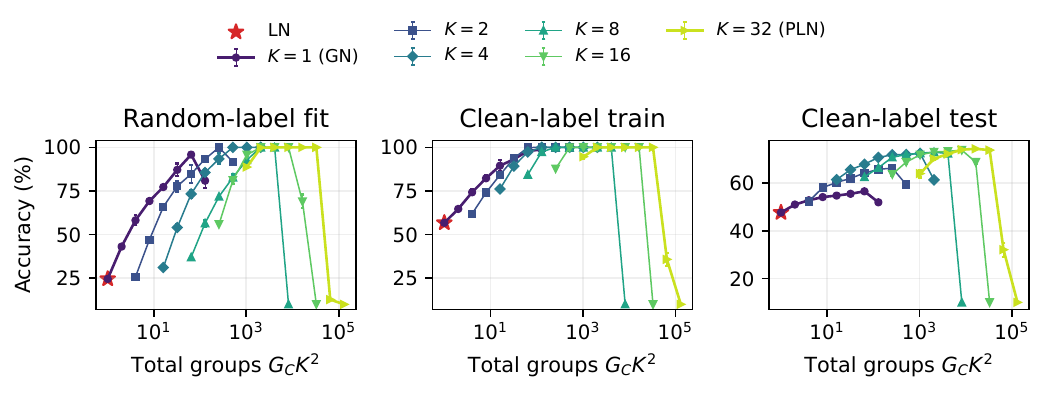}
    \caption{CNN accuracy versus total group count on CIFAR-10: (a) fitting
    $50{,}000$ fixed random labels, (b) clean-label training, and (c)
    clean-label testing. $K$ is the number of spatial patches per axis, so the
    total number of groups is $G_CK^2$. The GN ($K=1$) and position-wise PLN
    ($K=32$) directions are marked explicitly. Curves show three-seed means
    with sample standard deviations.}
    \label{fig:cnn_grouping_direction}
\end{figure}

\FloatBarrier

Figure~\ref{fig:cnn_grouping_direction} shows that the best attainable accuracy
generally rises as the total group count increases, until overly fine
partitions cause a sharp decline. Within a given grouping direction, accuracy
did not improve monotonically with group count, and the three settings favor
different grouping directions. At matched totals of 8--64 groups, GN
($K=1$) gives the highest random-label fitting accuracy among the evaluated
directions but the lowest clean-label test accuracy. For clean-label training,
GN leads at 8--16 groups but is overtaken at 32--64 groups. Thus group count
alone does not determine fitting or generalization; the grouped axes matter.



\FloatBarrier
\subsubsection{VGG-16 Initialization Sweep}
\label{sec:experiment 3}
\paragraph{Purpose and protocol.}
This auxiliary experiment studies whether the choice of nonlinear layer changes
the sensitivity of a deep convolutional network to its learning rate and weight
initialization. We train VGG-16 \citep{simonyan2014very} on CIFAR-10
\citep{krizhevsky2009learning}. We use the standard $50{,}000$-image training
split for optimization and the $10{,}000$-image test split only for held-out
evaluation. The VGG-16 backbone contains no additional BN or other normalization
layer: each hidden nonlinear layer is replaced by one of PLN[8], ReLU, SiLU,
Tanh, and SAT, while the rest of the architecture is held fixed. For a feature
map in $\mathbb R^{B\times C\times H\times W}$, PLN[8] independently normalizes
each consecutive group of eight channels at every spatial position. It uses
$\epsilon=10^{-5}$ and learnable channel-wise scale and bias, initialized to
one and zero, respectively. The same data split, training augmentation, and
optimization budget are used for every method, and no method-specific tuning is
performed outside the common grid described below.

\paragraph{Hyperparameter grid.}
For every nonlinear layer, we take the Cartesian product of the initial learning
rates
\[
    \{0.001,0.003,0.01,0.03,0.1,0.3,1\}
\]
and initialization multipliers
\[
    \{0.25,0.35,0.5,0.7,1.0,1.4,2.0,2.8,4.0\}.
\]
The multiplier rescales the He initialization \citep{he2015delving}; hence,
$1.0$ is the unmodified He initialization, whereas, for example, $4.0$
multiplies the initialized weights by four. This gives $63$ hyperparameter
configurations per nonlinear layer and $315$ configurations in total. Each
configuration is trained independently with three random seeds, for a total of
$945$ runs.

\paragraph{Optimization and evaluation.}
All models are trained for $200$ epochs with mini-batch SGD
\citep{amari1993backpropagation}, batch size $256$, momentum $0.9$, and weight
decay $10^{-4}$. The selected initial learning rate is linearly warmed up during
the first $10$ epochs and is subsequently divided by $2.5$ at epochs
$50$, $90$, $130$, and $170$. The same CIFAR-10 data-augmentation pipeline is
used for all configurations. Weight decay is applied uniformly to all trainable
parameters, including the affine parameters of PLN[8]. For each fixed triple of
nonlinear layer, learning rate, and initialization multiplier,
Figure~\ref{fig:vgg_results} reports the mean test accuracy over the three seeds.
Thus, every cell in the figure reports a fixed configuration; no best-of-grid
selection is applied.

\begin{figure}[!htbp]
    \centering
    \includegraphics[width=\linewidth]{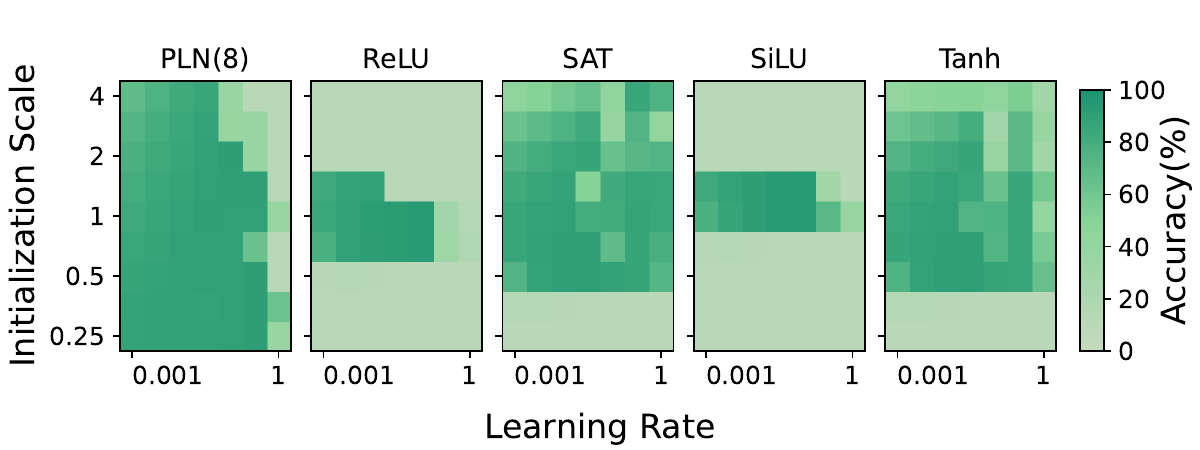}
    \caption{Mean CIFAR-10 test accuracy of VGG-16 models over three random seeds for
    each learning-rate--initialization pair. Each panel corresponds to one
    nonlinear layer, and darker green denotes higher accuracy. The learning rate
    increases from $0.001$ (left) to $1$ (right). The middle row uses the
    unmodified He initialization \citep{he2015delving}; the remaining rows scale
    its weights from $1/4$ (bottom) to $4$ (top).}
    \label{fig:vgg_results}
\end{figure}
\FloatBarrier

\paragraph{Interpretation.}
PLN[8] is comparatively less sensitive to initialization scale than the
element-wise activation baselines in this setting. We report this result as an
auxiliary optimization study.


\FloatBarrier
\subsection{ViT Normalization-Domain Experiments}
\label{sec:vit grouping details}

We train a depth-$4$, width-$192$ Vision Transformer
\citep{2021_ICLR_Dosovitskiy} on CIFAR-10 \citep{krizhevsky2009learning}.
Images are resized to
$224\times224$ and divided into $16\times16$ patches, giving $T=197$ tokens;
the model has three attention heads and FFN width $768$. We partition the token
and channel axes into balanced groups. The all-normalization grid uses
$G_T\in\{1,2,4,8,16,32,64,128,197\}$ and
$G_C\in\{1,2,4,8,12,24,48,192\}$; the activation-replacement grid uses the
same token group counts and
$G_C\in\{1,2,4,8,12,24,48,96,192\}$. In
all-normalization replacement every LN is replaced and GELU is retained; in
activation replacement every LN is retained and all FFN GELUs are replaced.

The primary controlled fitting study uses a fixed $6{,}000$-example subset
with $100\%$ random labels, following the finite-sample random-label fitting protocol
\citep{2017_ICLR_Zhang}, shared across model seeds $40,41,42$. Models
are trained for $100$ epochs using AdamW \citep{loshchilov2019decoupled}, batch
size $256$, constant learning rate
$10^{-4}$, and weight decay $5\times10^{-6}$. Affine parameters are enabled for
every normalization, while normalization weights and biases are excluded from
weight decay. We report the best accuracy evaluated on the same fixed
random-label set. The $20{,}000$-example study uses the same architecture and
optimization protocol, and extends the $6{,}000$ fixed examples and labels with
$14{,}000$ additional fixed examples and labels.

\begin{figure}[ht]
    \centering
    \includegraphics[width=\linewidth]{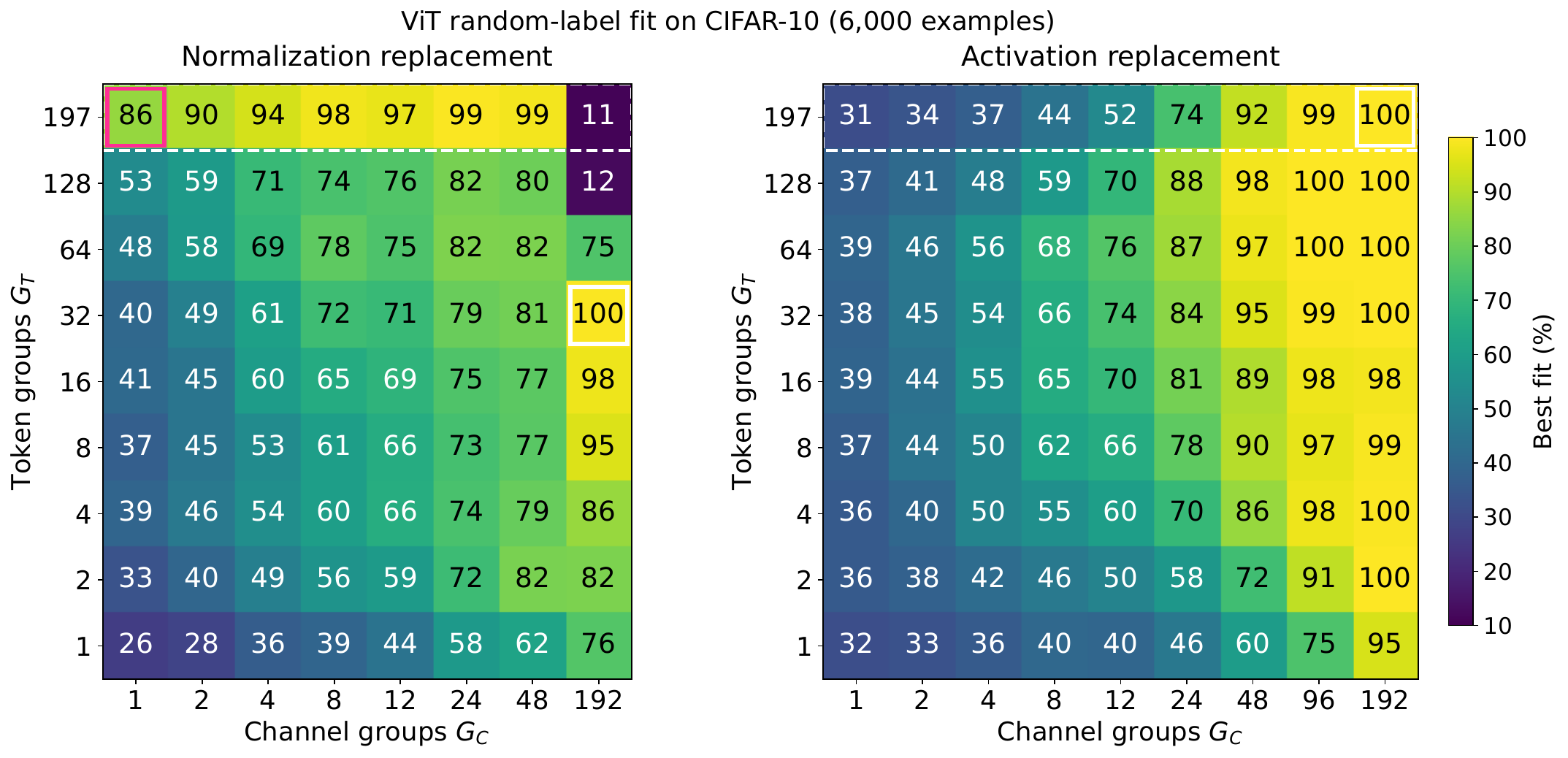}
    \caption{Random-label fitting over token--channel grouping domains in the
    depth-$4$, width-$192$ ViT. Each cell shows a three-seed mean rounded to
    the nearest percentage point.
    The dashed top row ($G_T=197$) is the position-wise PLN direction. Left:
    every LN is replaced while GELU is retained. Right: LN is retained while
    all FFN activations are replaced. The magenta box marks standard LN in the
    all-normalization panel, and white boxes mark panel maxima.}
    \label{fig:vit_random6k_grid}
\end{figure}

Figure~\ref{fig:vit_random6k_grid} shows that high random-label fitting
accuracy is concentrated at larger channel-group counts (the right side of
each heatmap). The effect of grouping across tokens depends on the replacement
site: under activation replacement, cross-token groupings ($G_T<197$) offer
a broader high-accuracy region than position-wise PLN, whereas under
all-normalization replacement the position-wise row is stronger at most
matched channel-group counts. Thus the bright right-hand region alone does
not establish a general advantage for cross-token grouping.

\paragraph{$20{,}000$-sample random-label results.}
Figure~\ref{fig:vit_random20k_grid} reports the complete three-seed grids. The
best all-normalization point reaches $89.49\%$ at
$(G_T,G_C)=(32,192)$, while the best activation-replacement point reaches
$91.66\%$ at $(64,192)$. In both modes, the best cross-token point substantially
outperforms the best point on the position-wise row.

\begin{figure}[ht]
    \centering
    \includegraphics[width=\linewidth]{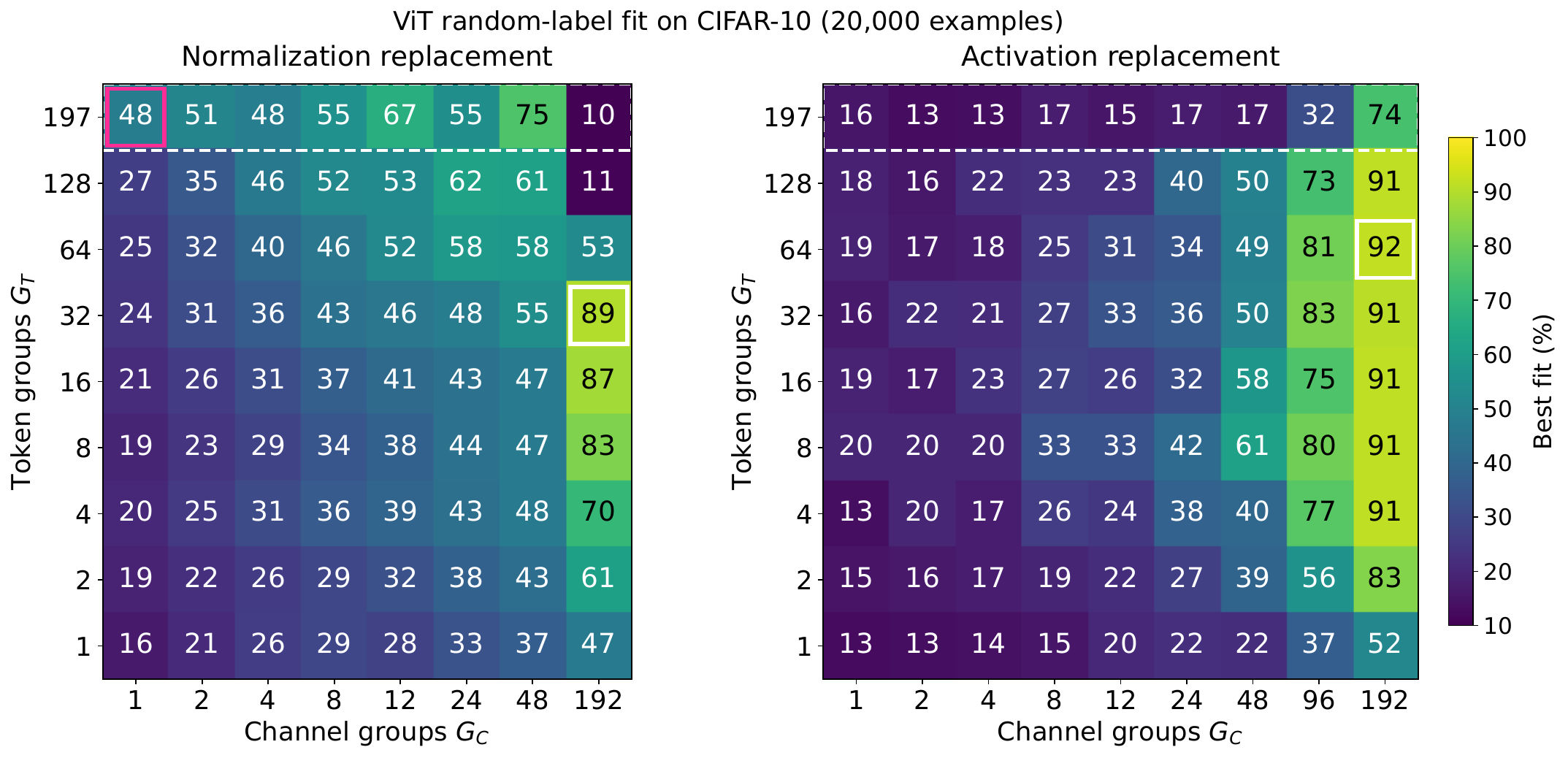}
    \caption{Three-seed random-$20{,}000$ ViT fitting results (cell labels
    rounded to the nearest percentage point). The magenta box
    marks standard LN in the all-normalization panel, the dashed row is the
    position-wise PLN direction, and white boxes mark maxima.}
    \label{fig:vit_random20k_grid}
\end{figure}

\paragraph{Clean-label comparison.}
For both replacement modes, we additionally train the complete grouping grid
on all $50{,}000$ clean CIFAR-10 training examples and evaluate on the standard
$10{,}000$-example test set. We use model seeds $40,41,42$ and otherwise retain
the architecture, $100$-epoch optimizer schedule, batch size, affine setting,
and weight-decay exclusions described above. Figure~\ref{fig:vit_clean_grid}
reports final training and test accuracy.

\begin{figure}[ht]
    \centering
    \includegraphics[width=\linewidth]{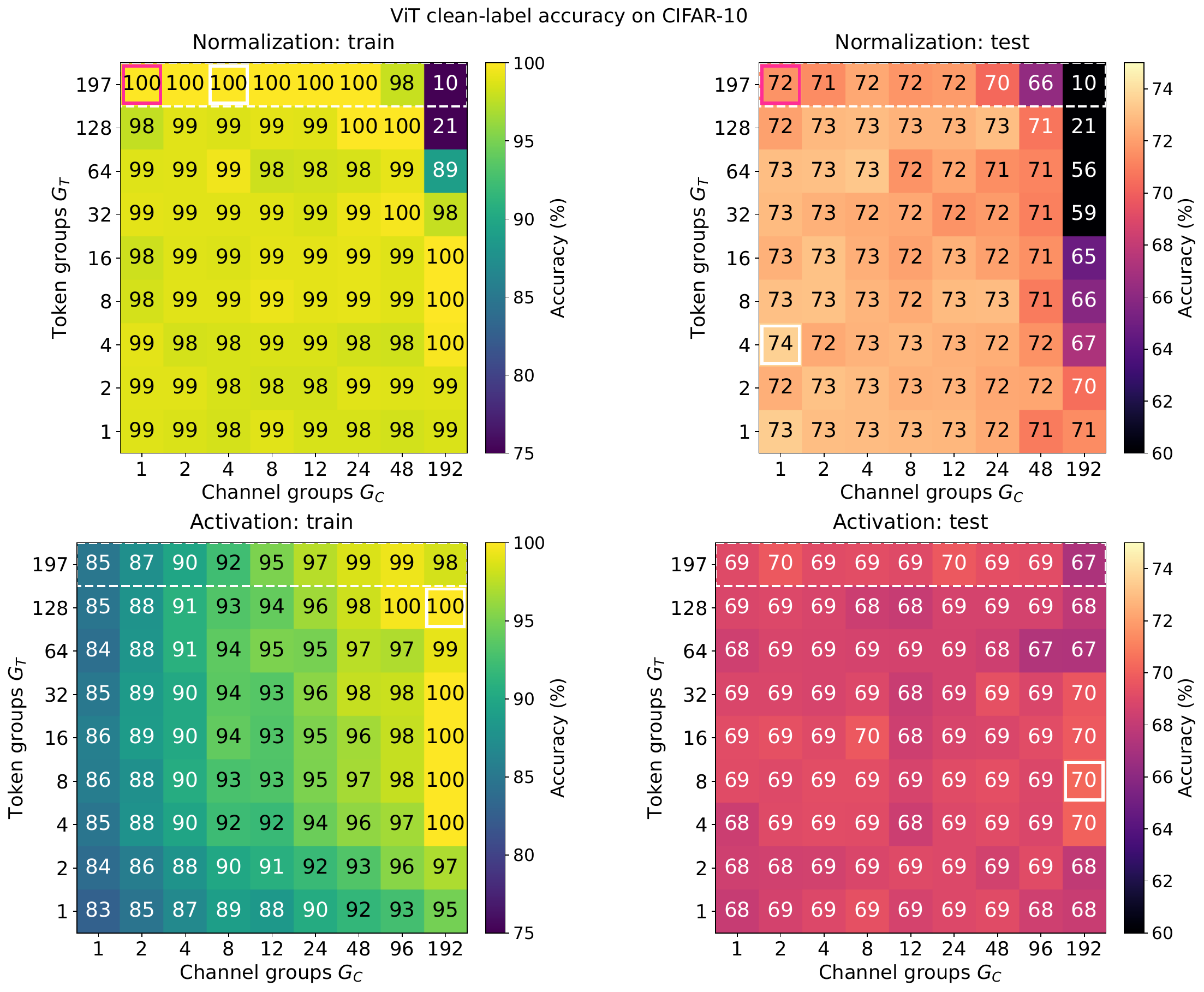}
    \caption{Three-seed clean-label ViT results (cell labels rounded to the
    nearest percentage point). Training accuracy varies
    strongly with domain geometry, while test accuracy occupies a much narrower
    range except at degenerate, near-scalar endpoints. The magenta box marks
    standard LN and white boxes mark panel maxima.}
    \label{fig:vit_clean_grid}
\end{figure}

\FloatBarrier

\newpage

\end{document}